\documentclass[final,12pt]{clear2025} 

\usepackage{censor}
\usepackage{comment}
\usepackage{amsmath}
\usepackage{amssymb}

\title[Extending Structural Causal Models for Autonomous Vehicles]{Extending Structural Causal Models for Autonomous Vehicles to Simplify Temporal System Construction \& Enable Dynamic Interactions Between Agents}
\clearauthor{%
 \Name{Rhys P. M. Howard} \Email{rhyshoward@live.com}\\
 \addr Oxford Robotics Institute, University of Oxford, 17 Parks Road, Oxford, OX1 3PJ, UK
 \AND
 \Name{Lars Kunze} \Email{lars@robots.ox.ac.uk}\\
 \addr Bristol Robotics Laboratory, T-Block, UWE Bristol, Bristol, BS16 1QY, UK\\
 \addr Oxford Robotics Institute, University of Oxford, 17 Parks Road, Oxford, OX1 3PJ, UK
}

\usepackage{bibentry}

\begin{document}
\maketitle

\begin{abstract}
In this work we aim to bridge the divide between autonomous vehicles and causal reasoning. Autonomous vehicles have come to increasingly interact with human drivers, and in many cases may pose risks to the physical or mental well-being of those they interact with. Meanwhile causal models, despite their inherent transparency and ability to offer contrastive explanations, have found limited usage within such systems. As such, we first identify the challenges that have limited the integration of structural causal models within autonomous vehicles. We then introduce a number of theoretical extensions to the structural causal model formalism in order to tackle these challenges. This augments these models to possess greater levels of modularisation and encapsulation, as well presenting temporal causal model representation with constant space complexity. We also prove through the extensions we have introduced that dynamically mutable sets (e.g. varying numbers of autonomous vehicles across time) can be used within a structural causal model while maintaining a relaxed form of causal stationarity. Finally we discuss the application of the extensions in the context of the autonomous vehicle and service robotics domain along with potential directions for future work.
\end{abstract}

\section{Introduction}
We continue to see ever greater integration of artificial intelligence (AI) into systems that physically interact with us and the world around us. This has coincided with the widespread emergence of deep neural networks and other black box methodologies that are limited in their ability to provide transparency and grounded explanations \citep{von2021transparency,adadi2018peeking,gunning2019darpa}. Explanations have long been tied to causality within philosophy \citep{salmon1998causality}, psychology \citep{gerstenberg2022would}, and statistics \citep{pearl2009causality}. Thus, in theory a high-resolution causal model-based approach to developing AI systems can allow us to engineer systems that are both explainable and transparent. Specifically, this work focuses upon autonomous agents that are embodied, and thus have to tackle the challenges that come with having a physical presence in the real world. This is in contrast to AI systems that operate upon or make decisions based upon data but which do not possess physicality, e.g. medical decision AI, automated stock-traders, chatbots, etc.

In this paper we aim to provide a series of extensions to causality literature to facilitate easier integration with autonomous vehicles (AVs). Such systems can potentially pose substantial risk to humans and / or valuable assets. Thus we propose structural causal model (SCM) \citep{pearl2009causality} integration can be utilised for explanation generation in the event of critical failure, both to establish accountability and to refine subsequent system design. Fig. \ref{fig:scenes} provides an example of the type of scenario in which needing to formulate grounded explanations is essential. SCMs also offer a high level of transparency through their structural equations and associated directed acyclical graphs (DAGs).

In order to better enable the integration of causality into AVs, we begin by examining the current state of literature in Sec. \ref{sec:related_work} as well as background on SCMs in Sec. \ref{sec:structural_causal_models}. We then discuss the specific challenges associated with SCM deployment in AVs in Sec. \ref{sec:scms_in_aess}, before proposing novel extensions to the SCM formalisms in Sec. \ref{sec:extensions} that address these challenges. Lastly we explore the extensions in the context of two domains and examine potential directions for future work in Sec. \ref{sec:discussion}.

In the ensuing work, we present the following contributions:
\begin{itemize}
    \item An exploration of literature and challenges surrounding the application of SCMs to AV agent scenarios.
    \item A series of extensions to SCM formalisms to better facilitate integration of SCMs with AI systems that are both embodied and autonomous.
    \item An examination of these extensions applied to the AV domain.
\end{itemize}

\begin{figure}[t]
    \centering

    \subfigure[Observed scene with collision]{
    \centering
    \includegraphics[angle=90,width=0.23\textwidth]{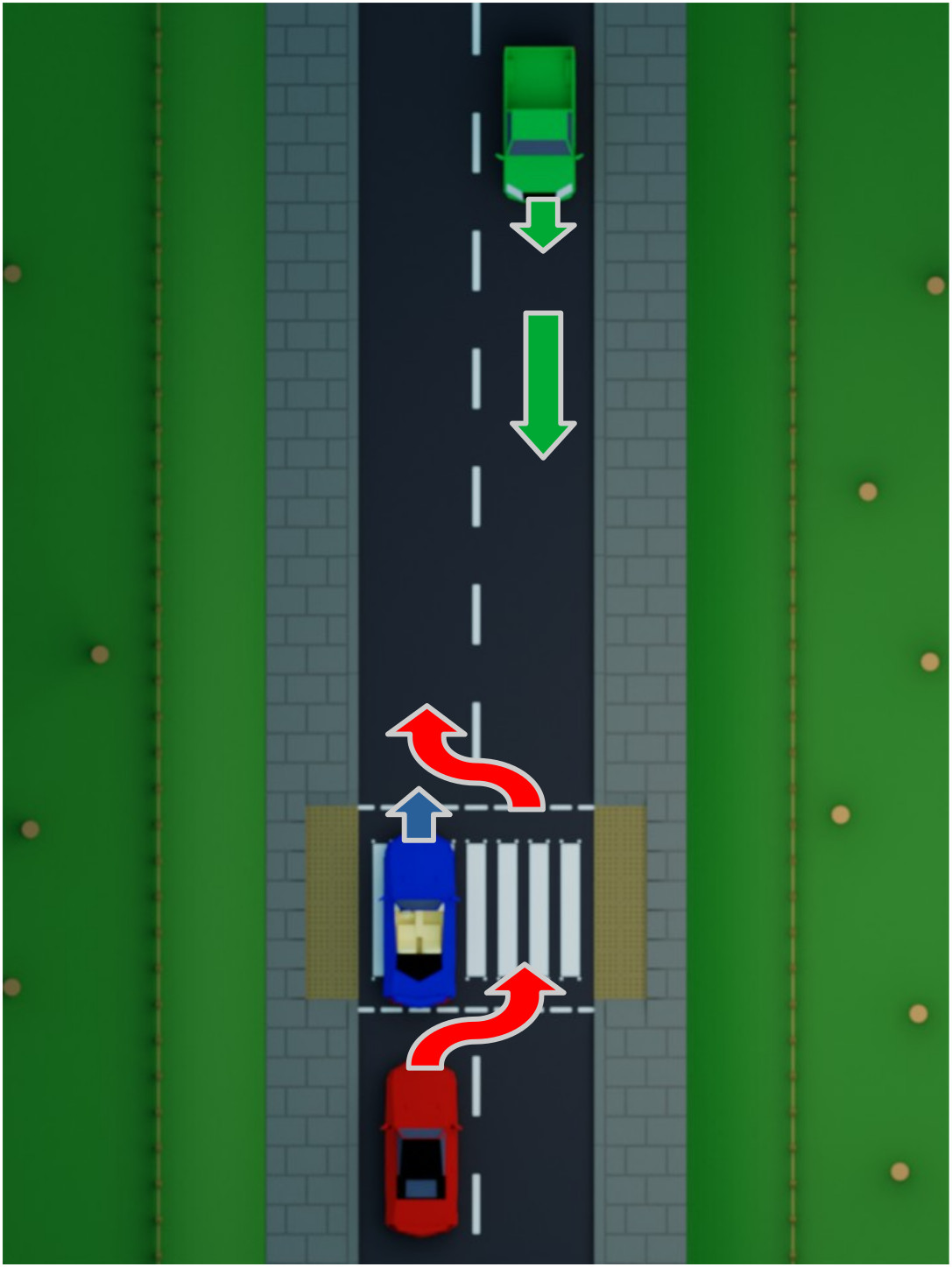}
    \includegraphics[angle=90,width=0.23\textwidth]{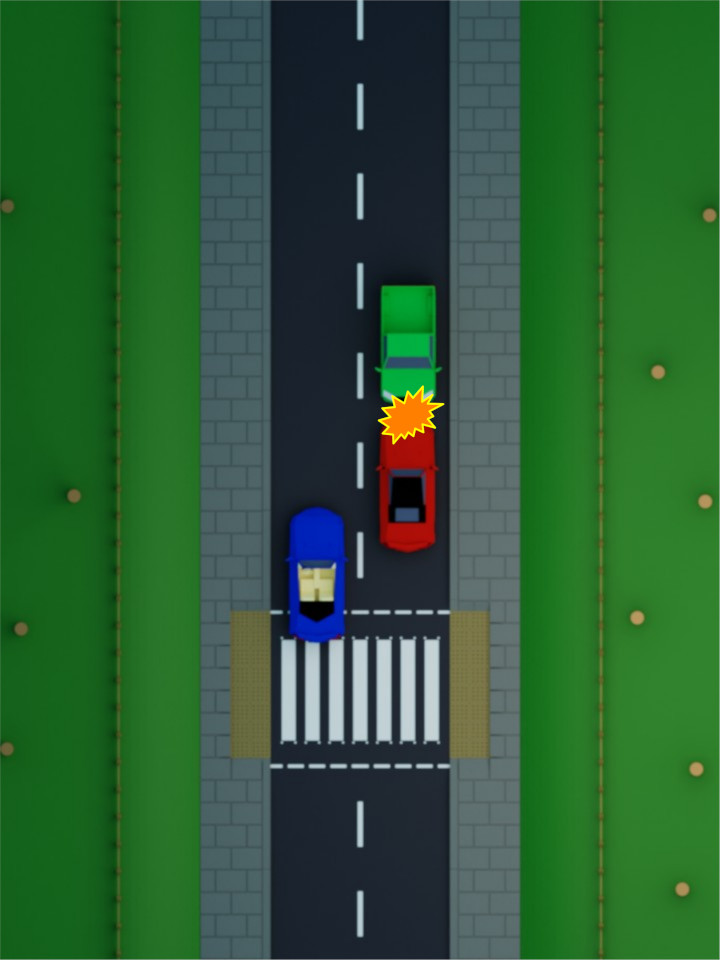}
    } \label{fig:scene_both}
    \subfigure[Counterfactual, no green speed-up]{
    \centering
    \includegraphics[angle=90,width=0.23\textwidth]{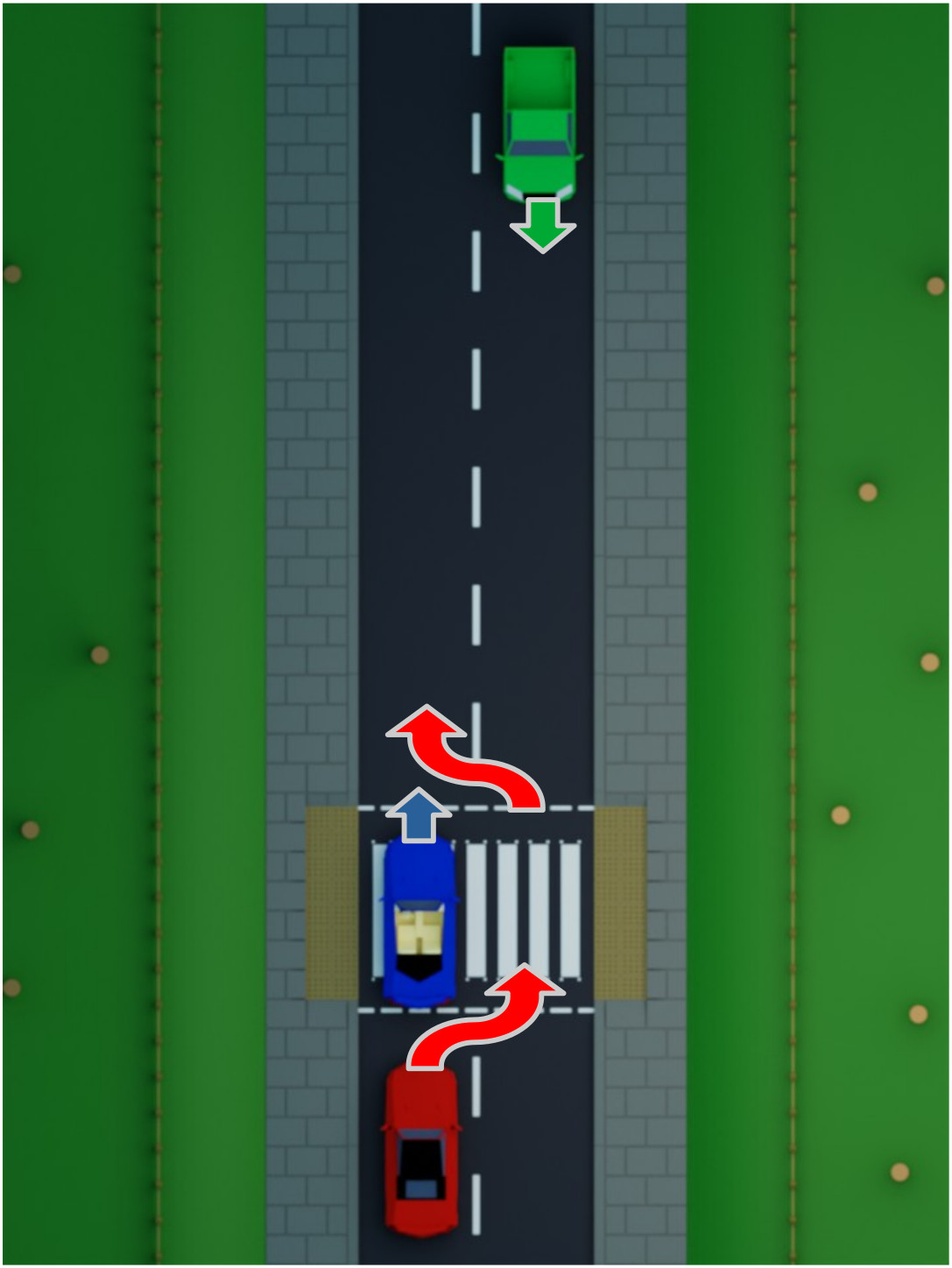}
    \includegraphics[angle=90,width=0.23\textwidth]{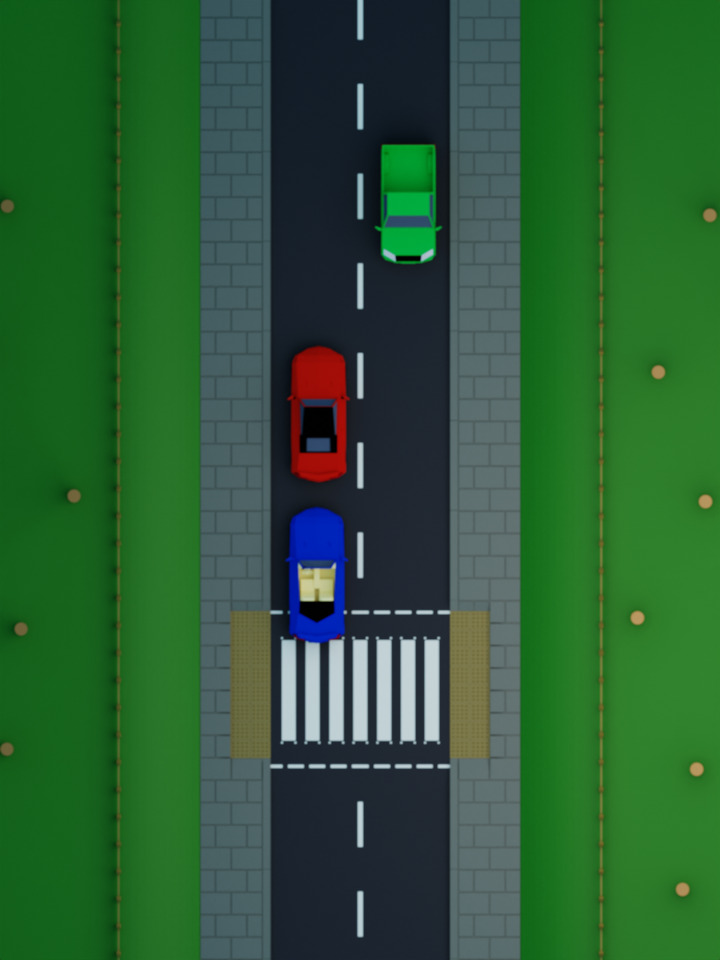}
    } \label{fig:scene_red}
    
    \subfigure[Counterfactual, no red overtake]{
    \centering
    \includegraphics[angle=90,width=0.23\textwidth]{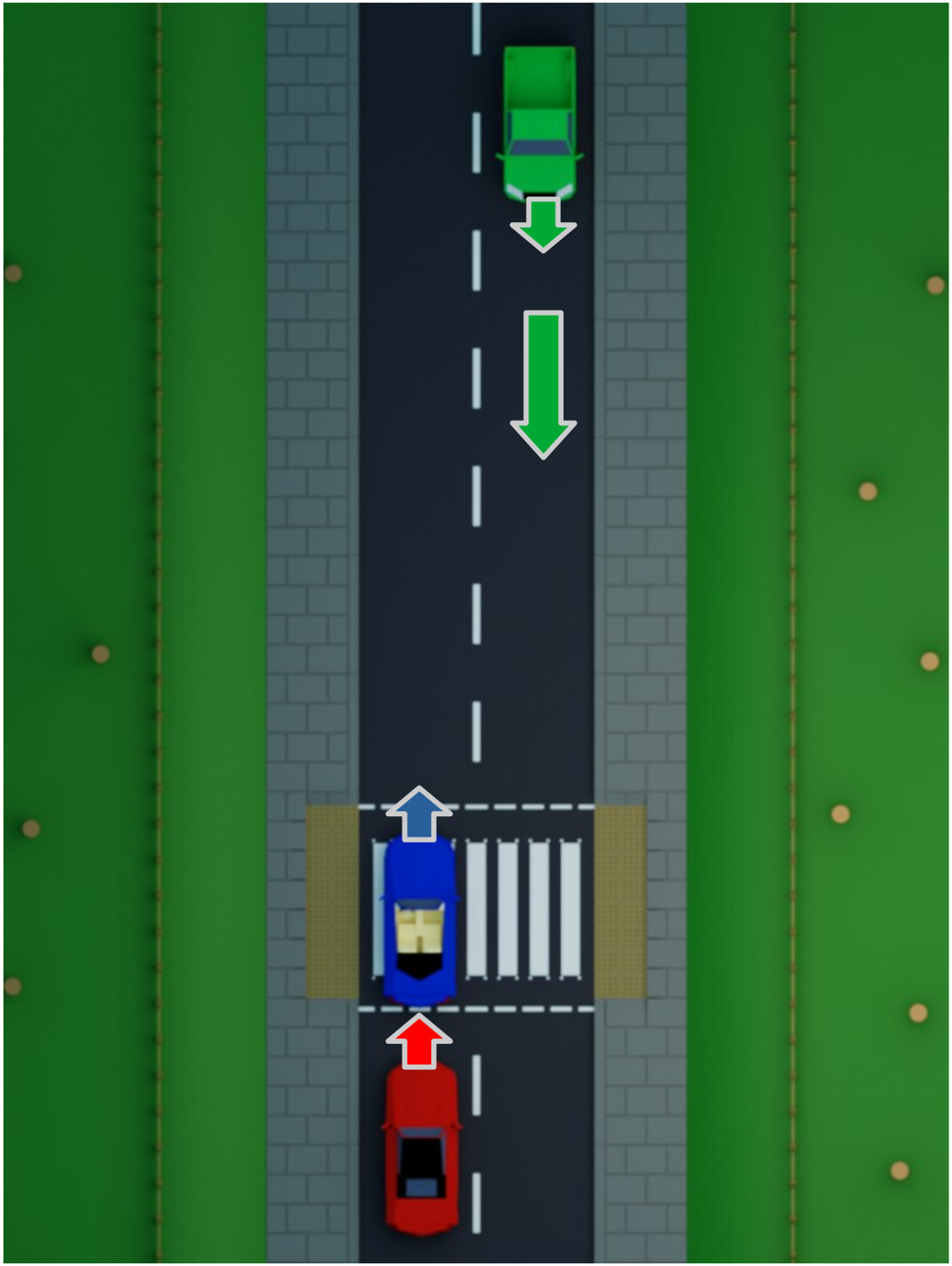}
    \includegraphics[angle=90,width=0.23\textwidth]{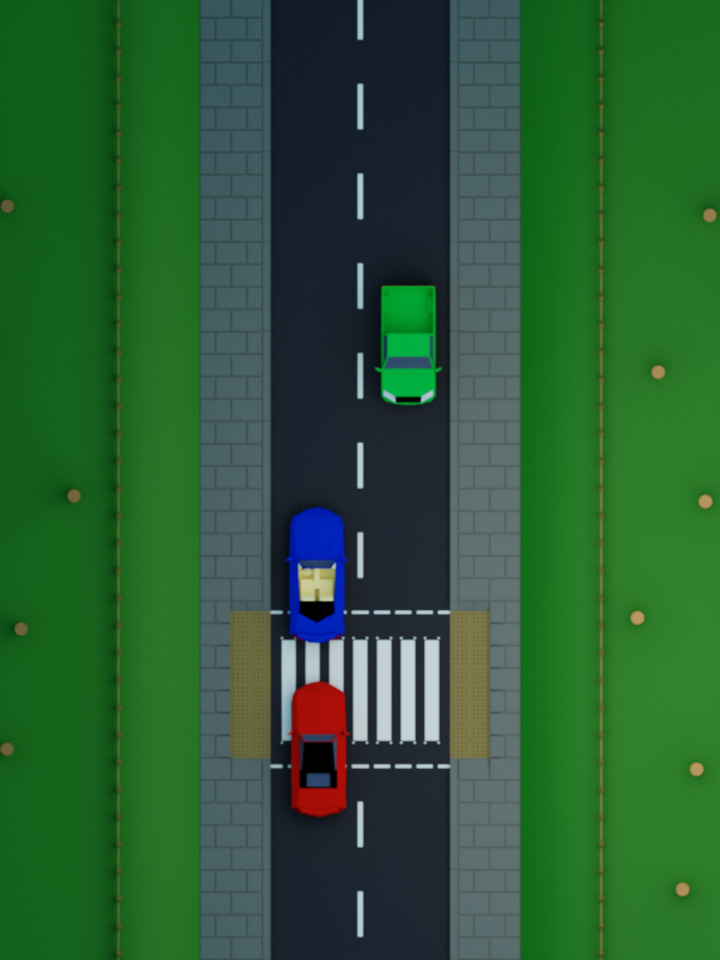}
    } \label{fig:scene_green}
    \subfigure[Counterfactual, neither action]{
    \centering
    \includegraphics[angle=90,width=0.23\textwidth]{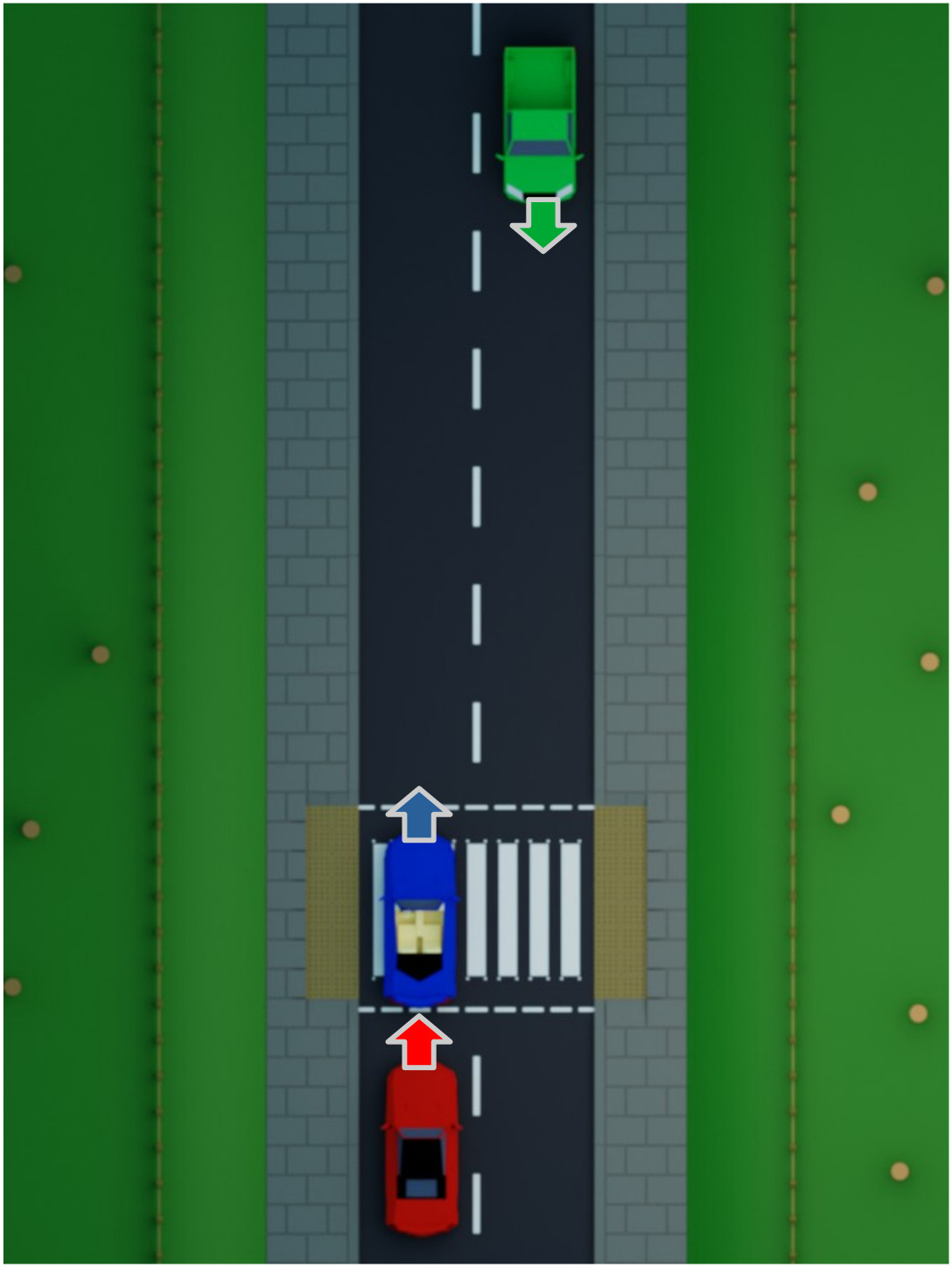}
    \includegraphics[angle=90,width=0.23\textwidth]{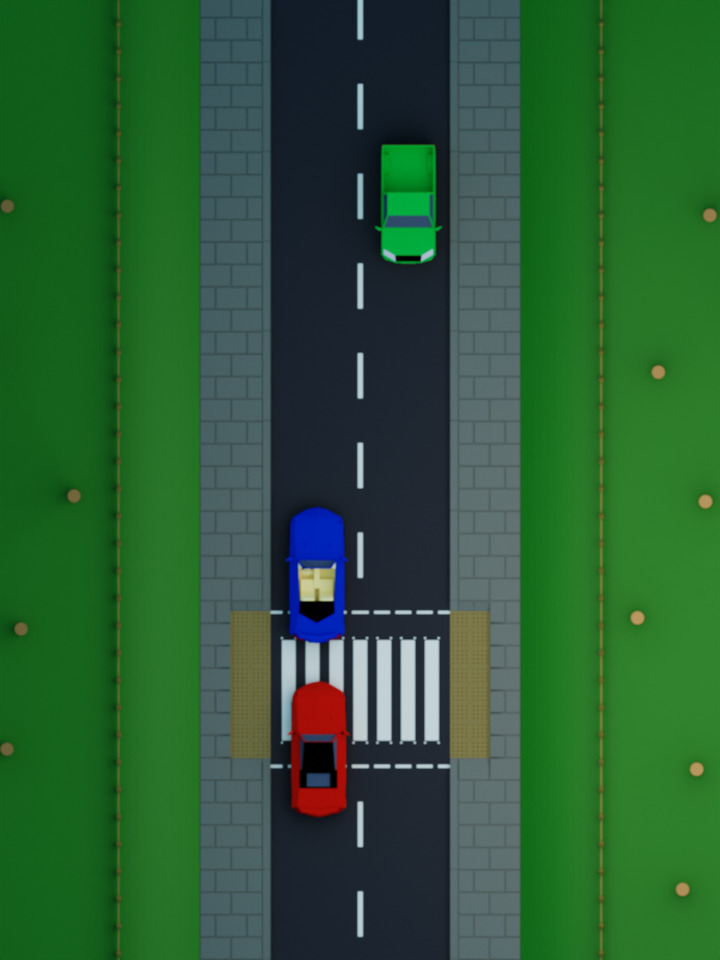}
    } \label{fig:scene_neither}

    \caption{Illustration of example fault attribution scenario. Subfigures show the initial scene with plan on the left and the final outcome on the right.}
    \label{fig:scenes}
    \vspace{-5mm}
\end{figure}

\section{Related Work} \label{sec:related_work}
As discussed in the previous section, causality has been identified as a critical component in explainable artificial intelligence (XAI) \citep{von2021transparency,adadi2018peeking,gunning2019darpa}, and more specifically explainable robotics \citep{hellstrom2021relevance}, allowing us to understand why decisions are made as they are. This in turn can be utilised in establishing accountability \citep{poechhacker2021algorithmic} or improving performance via self-extension \citep{marfil2024causalbased}. In this work we focus upon the integration of SCMs with AVs specifically due to the transparency they offer versus black-box methodologies such as neural networks (NNs). While causal modelling and discovery approaches have utilised methods such as variational autoencoders \citep{yang2021causalvae} and attention-based convolutional NNs \citep{nauta2019causal} to provide some level of explanation, they are at best able to ground such explanations in Granger causality \citep{granger1969investigating}. In an age where governments are moving to regulate artificial intelligence and legal cases surrounding it increase, we therefore aim to provide a means of constructing causally-grounded architectures with maximal transparency.

A promising direction for the integration of AVs and SCMs considers the integration of Markov decision processes (MDPs) with SCMs \citep{nashed2023causal,cannizzaro2023cardespot,triantafyllou2023towards}. MDPs are frequently used to model agents in AVs and thus these works represent a notable step in the right direction. However, of the aforementioned approaches of both \cite{nashed2023causal}, and \cite{cannizzaro2023cardespot} focus upon simple grid-world problems that concern only a single agent; while \cite{triantafyllou2023towards} do not focus upon embodied agents, and only capture agent interactions in terms of taking turns in two-player games.

Other approaches have utilised causal models similar to SCMs such as causal Bayesian networks for explanation, prediction and proactive behaviour in robotics \citep{diehl2023causalbased}, but these also focus upon a specific scenario considering a single agent. \cite{wich2022empirical} specify an SCM to identify key factors in hand selection when executing manipulation tasks, but once again said model is highly specialised to this purpose and only considers a single agent. \cite{madumal2020explainable} present a high-level SCM that seems to more generally capture agents playing \emph{StarCraft II} outside of a specific scenario, however each agent is represented via its own SCM rather than as part of a complete model and the system is disembodied. \cite{maier2024causal} provides a more detailed SCM with some modularity that even captures interactions between vehicles, the main limitation is that the SCM is specific to a simple advanced driver assistance system scenario.

There has also been some interest in representing ordinary differential equations through SCMs in order to capture dynamic systems \citep{mooij2013ordinary,bongers2021foundations}, however these methods have yet to be applied to autonomous vehicles to the best of our knowledge. These works achieve the representation of these equations in SCMs via the use of cycles, and for this reason we have opted to avoid utilising these techniques within this work, as our proposed formalisms and their associated example architectures are effectively acyclical --- a common requirement of many causal reasoning techniques.

We also set ourselves apart from work concerning the active discovery of causal relationships. \cite{castri2022causal} covers the discovery of causal relationships between agent variables within a multi-agent pedestrian environment. While such work can produce a causal graph, it is substantially more difficult to identify the structural equations pertaining to high-level relationships, and in some cases AV agent data may prove unsuitable for such techniques \citep{howard2023evaluating}. Instead we argue for utilising well established knowledge regarding kinematics / dynamics, control systems, and planning in order to design the causal graph and structural equations.

We argue this work provides the first comprehensive set of extensions that address challenges hampering the integration of SCMs into AVs. Although we cover these challenges in depth in Sec. \ref{sec:scms_in_aess}, obstacles relating to the temporal nature of AVs, as well as the lack of general system representation have been identified in previous works as reasons for avoiding SCM usage \citep{gyevnar2024causal}. In presenting our extensions, we aim to enable greater use of SCMs within AVs and foster discourse over remaining obstacles to such use.

\section{Structural Causal Models} \label{sec:structural_causal_models}
A popular means of representing causal relationships between variables is the SCM \citep{pearl2009causality,bareinboim2015bandits}. A SCM $M = \langle U, V, F, P(U) \rangle$ is comprised of both endogenous variables $V$ and exogenous variables $U$. Endogenous variables are those the model aims to describe and derive their values from the structural equations given in $F$. Meanwhile exogenous variables capture independent external factors that are not explicitly modelled, and thus derive their values from the joint probability distribution $P(U)$. We additionally introduce the function $Pa(\cdot)$ which gives the parent variables of an endogenous variable. In some cases the parents associated with a variable under discussion $V_i$ may have its parent variables $Pa(V_i)$ referred to as $V_{Pa}$ for brevity. 

\paragraph{Temporal Extension}
Many domains, including those relating to AVs make use of time series data or similar. Thus adaptations have been introduced to capture the temporal elements of causality in SCMs, with \cite{peters2017elements} giving a comprehensive overview of this in the context of temporal causal discovery. To summarise methods of temporally extending SCMs, we cover three common representations here:
\begin{itemize}
    \item \textbf{Summary Graph:} Utilises the original variables and establishes a causal link between variables if there is a causal link between said variables regardless of time lag. 
    \item \textbf{Full Time Graph:} Simply replicates each variable for the entire time period being modelled, thus allowing causal relations to be represented between any pair of variables each associated with any time step. 
    \item \textbf{Window Graph:} Has $(\tau + 1)$ copies of each variable in order to represent causal relations with a maximum time lag of $\tau$. This relies upon the causal relations between the relative time steps of the variables within the window graph being consistent across all time steps modelled, i.e. causal stationarity \citep{runge2018causal}. 
\end{itemize}

\section{Problem Definition}
\subsection{Autonomous Vehicle Architecture}

\begin{figure*}[t]
    \centering
    
    \subfigure[Overview of Autonomous Vehicle Architecture]{
        \includegraphics[width=0.63\linewidth]{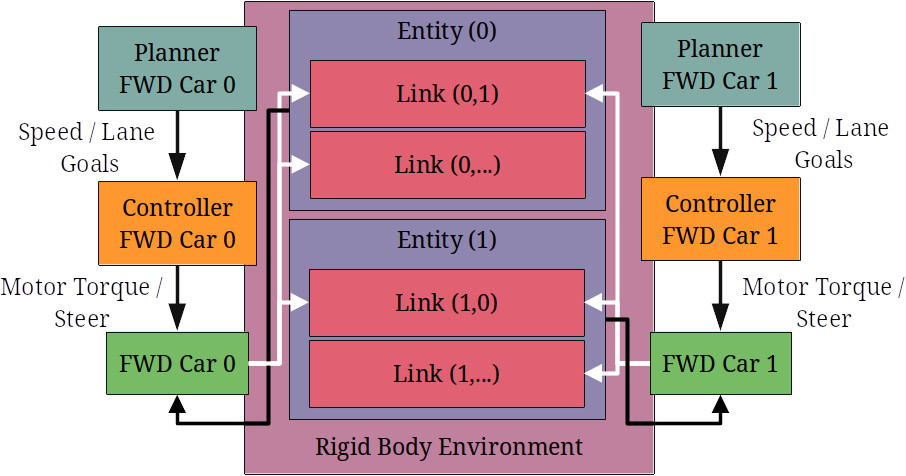}
        \label{fig:overall_architecture}
    }
    \subfigure[A\,\,Fault\,\,Attribution\,\,Graph]{
        \includegraphics[width=0.32\linewidth]{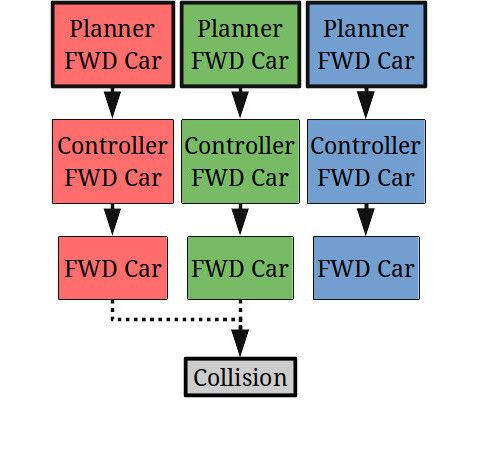}
        \label{fig:architecture_scenario}
    }
    
    \caption{The left side depicts the overall AV SCM-based architecture, including AV agent interactions via links. The right side depicts the components of this architecture and how they interact causally with the collision scenario depicted in Fig. \ref{fig:scenes}.}
    \label{fig:both_architecture_figures}
    \vspace{-5mm}
\end{figure*}

Our primary motivation behind developing an SCM architecture capturing the interactions of AVs in driving scenarios is for post-hoc analysis. Critically, driving presents a safety-critical domain that nonetheless presents a greater degree of order in terms of the expectation that road agents follow the laws of the road. The combination of these traits makes it both appealing and convenient to model causally. An overview of the autonomous vehicle SCM architecture is depicted in Fig. \ref{fig:overall_architecture}.

The architecture represents individual AV agents via a planner object to provide high-level goal values, a controller to translate those goal values into continuous actuation values, and an object representing the embodied system that takes the actuation values as input. The embodied system representation object is further connected to an entity object that represents the embodied system within an environment shared with other agents. These entity objects in turn store a number of link objects that each capture interactions between the entity embodied system and another embodied system.

\subsection{Example Scenario} \label{sec:example_scenario}

An example vehicular scenario of the beneficial utilisation of SCM integration in post-hoc analysis is provided in Fig. \ref{fig:scenes}. The scene depicts a red vehicle overtaking and a green vehicle accelerating from the opposite direction, resulting in a collision. The causal links related to this are depicted in Fig. \ref{fig:architecture_scenario}. This shows both the actions of the red agent and the green agent were necessary to cause the collision, and while failures could in theory occur at the mechanical or controller level, we are interested in decisions made by the agent planners. 

Legally speaking the exact proportion of blame one can attribute to each of the agents will of course depend upon the temporal ordering of actions, the laws of the jurisdiction in which the crash occurred, and other matters of context. However, this work focuses purely on identifying how causal techniques can be used to establish agent action attribution rather than discussing the extent of culpability. We approach this by tying measures of causal necessity to the `but-for' test as laid out in the Model Penal Code \citep{ali1962model} in the manner described by \cite{pearl2018book}. This also bears resemblance to the form of actual causality between events as described by \cite{halpern2016actual}, albeit in the naive form without considering counterfactual contingencies over variables.

\subsection{Challenges for SCM Integration in Autonomous Vehicles} \label{sec:scms_in_aess}

\subsubsection{Modularisation \& Encapsulation} \label{subsec:modularisation_and_encapsulation}
While SCM use in AVs remains rare, those works that do utilise SCMs typically either utilise a single monolithic SCM, or several entirely independent SCMs. Furthermore it is rare to store the data relating to an SCM within the structure of an SCM itself.
SCMs are to some degree already modular by nature \citep{pearl2009causality}, as a causal graph may be cut and endogenous variables replaced with exogenous. However, this leaves no mechanism for separating SCM modules once connected, nor a means to control exogenous variable access. 

\subsubsection{Constant Space Representation of Time Series}
Window graphs offer a compact representation of temporally lagged causal relationships, particularly for AVs, which often base their current time step state purely upon their previous time step state --- i.e. they are memoryless. However, even the window graph representation requires a roll-out during inference to cover the relevant time period. This leads to the rolled-out graph taking up more space than the initial window graph, not to mention that inferences over different time periods usually each require their own independent roll-outs. This becomes especially complicated when one wishes to use the generative properties of SCMs to counterfactually simulate different outcomes, creating separate branches for each inferred history. 



\subsubsection{Arbitrary Number of Agent Interactions}
When operating within their environments AVs will interact with various agents. The number and nature of these agents will often vary, thus an SCM capturing an AV should allow the handling of dynamic set of values. While the structural equations of an SCM can derive an output from an input set, the greater issue is being able to dynamically formulate such a set within an SCM. In instances where this set is derived directly from observation, this can simply be captured via an exogenous variable. However, we are considering a dynamic set of agents we may also wish to model within the SCM, and this needs to be able to vary with time.

\section{Extending Structural Causal Models for Autonomous Vehicles} \label{sec:extensions}

\subsection{Structural Equations with Side-Effects}
Before we can properly introduce any extensions to the SCM formalism, we first need to describe the process of integrating side-effects into the structural equations of an SCM. Side-effects in this context of this work refer to computations that occur beyond the purely mathematical calculations of structural equations (e.g. I/O, memory). In order to incorporate this into the SCM formalism, we take a similar approach to that typically observed in functional programming paradigms by utilising monadic values \citep{wadler1992essence}. The use of monads is widely documented and thus only a succinct description is given here. A monad is comprised of: a type constructor $\mathfrak{M}(\cdot)$; a unit / return function $u_\mathfrak{M}(\cdot): \mathcal{T} \rightarrow \mathfrak{M}(\mathcal{T})$ which wraps a value of type $\mathcal{T}$ with the monad type $\mathfrak{M}$; and a bind function $b_\mathfrak{M}(\cdot): \mathfrak{M}(\mathcal{T}) \rightarrow (\mathcal{T} \rightarrow \mathfrak{M}(\mathcal{T}^\prime)) \rightarrow \mathfrak{M}(\mathcal{T}^\prime)$ that applies a function to the wrapped value of the monad while propagating the monad forwards. The bind function is also occasionally denoted via a binary operator $\triangleright_\mathfrak{M}$ for a cleaner presentation.

\subsection{Variable Context} \label{subsec:variable_context}
A variable context $C$ is a set of meta-variables that are globally accessible by the structural equations of the SCM to which it belongs. Utilisation of such a concept must be handled with care, as one could inadvertently incorporate causal relationships via $C$ that would not be represented within the causal graph of the SCM. However, $C$ only captures the current time $C_T \in C$ and the time step size $C_{\delta t} \in C$, and these are only used to emulate a temporal roll-out and carry out other time-related computations. Thus $C$ does not invalidate SCMs that make use of it provided they do so via the structural equations described in Sec. \ref{subsec:temporal_variables}.

In theory we can utilise the well documented state monad \citep{wadler1992essence} --- denoted with $\mathfrak{S}$ where necessary --- to track the values of $C$ during inference. However, for the sake of simplicity in notation, and in order to focus upon the core contributions of this work, we will assume that $C$ is globally accessible and thus reference its meta-variables directly.

\subsection{Temporal Variables} \label{subsec:temporal_variables}
Temporal variables do not refer to a specific variable or structural equation but instead a class of endogenous variables that operate on, or with, time meta-variables. These can be further separated into variables whose computations rely upon meta-variables and previous time step (PTS) variables which establishing the temporal structure of the SCM. We will discuss the latter of these first, as they are essential to understanding temporal representation within the extended SCM formalism.

\subsubsection{Previous Time Step Variables}
We define a PTS variable as a type of endogenous variable $V_\delta \in V$ associated structural equation $f_{V_\delta} \in F$ defined as follows:
\begin{equation}
    f_{V_\delta}(V_{Pa}) = \delta_-\ \triangleright_\mathfrak{S}\ (\lambda\emptyset.\ V_{Pa})\ \triangleright_\mathfrak{S}\ \delta_+
\end{equation}
where $\emptyset$ denotes the empty set / value,  $\delta_- : \mathfrak{S}(\emptyset)$ is a monadic value, and $\delta_+(\cdot) : \mathcal{T}_{V_{Pa}} \rightarrow \mathfrak{S}(\mathcal{T}_{V_{Pa}})$ is a function. Importantly $\delta_+$ and $\delta_-$ utilise the action associated with their monadic value to increment and decrement the current time $C_T$ of $C$ by $C_{\delta t}$ respectively. This has the effect of shifting the current time to the previous time step, evaluating the parent variable $V_{Pa}$, before returning the current time to where it was initially. The function expects that the parent $V_{Pa}$ also possesses the monad $\mathfrak{S}$. If this is not already the case --- i.e. from $V_{Pa}$ having PTS variable ancestors of its own --- this can easily be achieved via use of the unit function $u_\mathfrak{S}(\cdot)$.

\begin{figure}[t]
    \centering

    \subfigure[Window Graph SCM]{
    \includegraphics[width=0.35\linewidth]{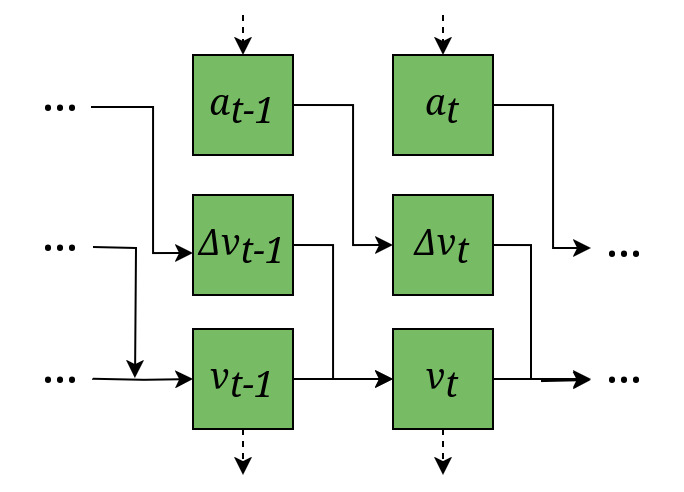}
    } \label{fig:time_series_scm}
    \ 
    \subfigure[Extended SCM]{
    \includegraphics[width=0.2\linewidth]{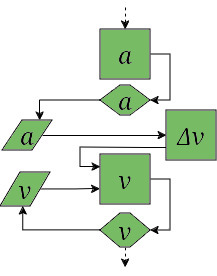}
    } \label{fig:system_scm}
    \ 
    \subfigure[Legend]{
    \includegraphics[width=0.28\linewidth]{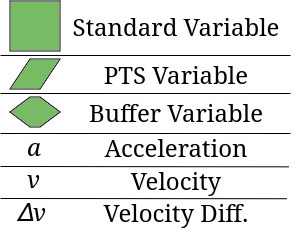}
    } \label{fig:kinematics_legend}

    \caption{The SCMs depict the calculation of velocity $v$ from acceleration $a$ by calculating the velocity difference $\Delta v$ and adding it to the velocity from the previous time step. A typical time series SCM would be coupled along with time series data, corresponding with time indexed variables --- indicated via subscripts. In contrast the extended SCM representation uses PTS variables and buffer variables. These allow the SCM a fixed-size graph while still capturing causal links between time steps, in addition to providing data encapsulation.}
    \label{fig:scm_structures}
\end{figure}

By utilising PTS variables we can represent temporal causal relationships in an SCM while maintaining the size and shape of the causal graph even during inference. Fig. \ref{fig:scm_structures} illustrates this by showing how temporal causal relations in kinematics can be captured via the introduction of recursive cycles of the causal graph, as opposed to performing a roll-out of a window graph during inference. Typically cycles would invalidate the SCM and break with assumptions made by many causal reasoning techniques. However the combination of graphical structure, PTS variables, and $C$ means that inference upon the SCM effectively emulates a roll-out while avoiding any modification to the SCMs graphical structure. Thus any casual reasoning method which can operate upon a window graph SCM should equally function upon an SCM utilising the proposed temporal representation (see App. \ref{app:pts_variable_recursion_proof} for proof). These variables are primarily utilised in the mechanical dynamics calculations that comprise the \emph{FWD Car}\,\footnote{Front-Wheel Drive (FWD)} component of the architecture depicted in Fig. \ref{fig:both_architecture_figures}.

Note that while the graphs we present do bear some resemblance to the extended summary graphs described by Assaad et al. \cite{assaad2022discovery} they are arguably closer to a memoryless window causal graph. Extended summary graphs split their variables into a `present' set, and an all-encompassing `past' set. Whereas our approach can be seen as splitting variables into a `present' set, and a `previous' set, which only concerns the previous time step and no further back. However due to our representation being defined recursively it can still encompass the whole history modelled by the SCM. The key difference here is that extended summary causal graphs eschew causal stationarity while we rely upon its presence. Similarly the memoryless nature of the graphs shares similarities to past work modelling dynamical systems via local independence graphs \cite{mogensen2018causal}. However, this work explicitly defines the structural equations of its SCMs and their randomness is captured via exogenous variables. Hence the techniques often associated with local independence --- which typically model Markovian processes --- are unsuitable, as all structural equations are deterministic, and exogenous variables are by definition independent of one another. Thus, in a sense a local independence graph representation is much closer to a summary graph than what we describe here.

\subsubsection{Time Step Size Product / Quotient Variables}
A time step size product (TSSP) variable $V_{\cdot\delta} \in V$ or time step size quotient (TSSQ) variable $V_{\div\delta} \in V$ simply aims to multiply or divide their input by the time step size $C_{\delta t}$ respectively. Formally, we can define their functions as follows:
\begin{equation}
    f_{V_{\cdot\,\,\!\delta}}(V_{Pa}) = V_{Pa} \,\cdot\,C_{\delta t}
\end{equation}
\begin{equation}
    f_{V_{\div\delta}}(V_{Pa}) = V_{Pa}\div C_{\delta t}
\end{equation}
Again these are primarily useful for carrying out the types of kinematics and dynamics calculations shown in Fig. \ref{fig:scm_structures} and utilised by the mechanical dynamics \emph{FWD Car} components of Fig. \ref{fig:both_architecture_figures}. While linear scaling may not be suitable for modelling all embodied system variables, it suffices in the context of this work for modelling AV dynamics in a similar manner to Kalman filtering.

\subsubsection{Time -- Current Time Difference Variables}
The slightly more complex time -- current time difference (TCTD) type variable $V_{\delta T} \in V$ calculates the duration between the input of the variable and the current time $C_T$ as follows:
\begin{equation}
    f_{V_{\delta T}}(V_{Pa}) = V_{Pa} - C_T
\end{equation}
TCTD variables are primarily useful within the \emph{Controller FWD Car} modules shown in Fig. \ref{fig:both_architecture_figures}, where they can be utilised to determine the duration of time between now and when a planned action goal value should be achieved by.

\subsubsection{Time Conditional Variables}
Lastly we introduce the time conditional type variable $V_{T?} \in V$ that take the current time $C_T$ and compare it against a predetermined time in order to select which parent from which to derive its output. Their structural equations are defined as follows:
\begin{equation}
    f_{V_{T?}}(V_{Pa}^0, V_{Pa}^1) = \begin{cases}
        V_{Pa}^0 & C_T < \theta_T\\
        V_{Pa}^1 & C_T \geq \theta_T
    \end{cases}
\end{equation}
where $\theta_T$ is a configurable time parameter specified during structural equation construction. Time conditional variables have multiple uses including dynamic introduction / removal of agents in environment interaction calculations. Furthermore though, they can be utilised for counterfactual simulation purposes, as one can effectively splice together a new counterfactual SCM together with an original SCM. From here $\theta_T$ can be used to dictate the divergence point between the original and counterfactual SCM histories. This is primarily used by the \emph{Planner FWD Car} modules of Fig. \ref{fig:both_architecture_figures} in order to envisage a series of alternate actions diverging at each action start time with the use of a time conditional variable.

\subsection{Buffer Variables} \label{subsec:buffer_variables}
Having introduced the variable context and explored how one can utilise temporal variables to interact with the meta-variables of $C$ we can now utilise this to store time-indexed variable data encapsulated within the SCM. A buffer variables $V_B \in V$ has a single parent, and as the name suggests they act as a buffer for this parent, storing data relevant to it. In order to do so a buffer variable $V_B$ maintains a time-indexed dictionary $D_{V_B} = \langle T_{V_B}, d_{V_B} \rangle$ where $T_{V_B} \subset T$ is the set of time steps with indexed data, and $d_{V_B}(\cdot): T \rightarrow \mathcal{T}_{V_B}$ is a function that maps a time step to a value or distribution for variable $V_B$ of type $\mathcal{T}_{V_B}$. With this established the structural equation for a buffer variable $f_{V_B} \in F$ is defined as follows:
\begin{equation}
    f_{V_B}(V_{Pa}) = \begin{cases}
        \text{\textsc{up}}_{\{\}}(d_{V_B}(C_T)) & C_T \in T_{V_B}\\
        \text{\textsc{up}}_{\{(D_{V_B}, C_T, V_{Pa})\}}(V_{Pa}) & C_T \not\in T_{V_B}
    \end{cases}
\end{equation}
Here $f_{V_B}(\cdot): \mathcal{T}_{V_{Pa}} \rightarrow \mathfrak{B}(\mathcal{T}_{V_{Pa}})$ either returns the output of its parent variable, or retrieves stored data, depending upon whether the current time $C_T$ is found within the dictionary time steps $T_{V_B}$. Importantly the output type of $V_B$ matches that of its parent $V_{Pa}$, denoted as $\mathcal{T}_{V_{Pa}}$, albeit with the addition of wrapping the type in the buffer monad $\mathfrak{B}$. The buffer monad takes a single type and provides a data constructor $\text{\textsc{up}}_{x}(\cdot)$ (i.e. update with $x$). The data constructor takes a set $x$ of tuples structured $(D, t, V)$ and for each tuple stores the output of $V$ at tuple creation
within the dictionary $D$ time-indexed to time step $t \in T$. In order to comply with monad requirements, we provide the following unit and bind functions:
\begin{equation}
    u_\mathfrak{B}(V) = \text{\textsc{up}}_{\{\}}(V)
\end{equation}
\begin{equation}
    b_\mathfrak{B}(\text{\textsc{up}}_x (V), f) = f(V)\,\triangleright_\mathfrak{B}\,(\lambda \text{\textsc{up}}_{x^\prime} (V^\prime).\,\text{\textsc{up}}_{x\,\cup\,x^\prime} (V^\prime)) 
\end{equation}
This monad bears resemblance to the writer monad \citep{jones1993composing} albeit with the addition of time-indexed writing to the specified buffer variable dictionaries.

Once again, we refer back to Fig. \ref{fig:scm_structures} to depict how buffer variables can be used to encapsulate SCM data within modules, in this case kinematics data. This is utilised throughout the components of Fig. \ref{fig:both_architecture_figures} in order to dictate where data should be stored for variables as an inherent part of the SCM structure.

\subsection{Socket Variables} \label{subsec:socket_variables}
As mentioned in Sec. \ref{subsec:modularisation_and_encapsulation} SCMs are inherently well suited to modularisation \citep{pearl2009causality}. Since exogenous variables capture factors outside of the model that nonetheless influence the model, if one then produces a model for these factors they can easily combine the SCMs together to form a new greater whole. The main limitations of the naive combination of SCMs is the lack of ability for system designers to set limits on how SCMs are joined, and the difficulty in dynamically separating SCMs.

\begin{figure*}[t]
    \centering
    
    \subfigure[SCM without Socket Variables]{
        \includegraphics[width=\linewidth]{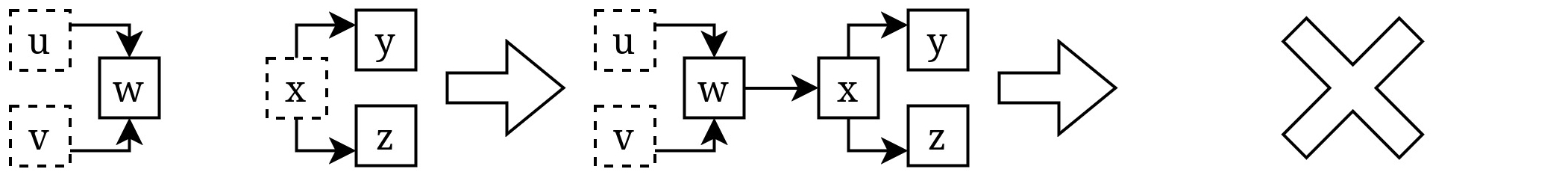} \label{fig:scm_merge_base}
    }
    
    \subfigure[SCM with Socket Variables]{
        \includegraphics[width=\linewidth]{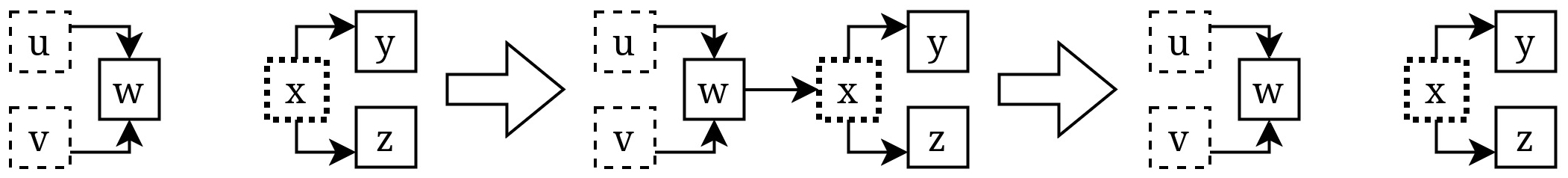} \label{fig:scm_merge_socket}
    }
    
    \caption{Illustration of merging and un-merging SCMs without and without socket variables. Here solid borders denote endogenous variables, dashes borders exogenous variables, and dotted borders socket variables.}
    \label{fig:scm_merge}
\end{figure*}

Thus we decompose the exogenous variables $U$ into socket variables $U^S$ and hidden exogenous variables $U^H$. A socket variables $U^S_i \in U^S$ differs from a typical exogenous variable in that its value is conditionally derived as follows:
\begin{equation}
    U^S_i \leftarrow
    \begin{cases}
        P(U^S_i) & |Pa(U^S_i)| = 0\\
        Pa(U^S_i) & |Pa(U^S_i)| = 1
    \end{cases}
\end{equation}
This allows one to join SCMs by assigning $U^S_i$ to take the output of a variable from another SCM. Importantly tracking socket variables separately ensures that the act of joining SCMs is reversible, allowing the dynamic reconfiguration of SCM-based modules. Note that the above function only accounts for cases of zero or one socket variable parents. This is because socket variables impose a maximum of one parent variable as a structural constraint on SCMs --- similar to endogenous variables allowing a maximum of one exogenous parent.

Fig. \ref{fig:scm_merge} illustrates this process. Here if we attempt to utilise SCMs without socket variables (see Fig. \ref{fig:scm_merge_base}), we merge SCMs $\langle U=\{u,v\}, V=\{w\}, F, P(U) \rangle$ and $\langle U^\prime=\{x\}, V^\prime=\{y,z\}, F^\prime, P(U^\prime) \rangle$ to get a new SCM $\langle U^{\prime\prime}=\{u,v\}, V^{\prime\prime}=\{w,x,y,z\}, F^{\prime\prime}=F\cup F^\prime, P(U^{\prime\prime}) \rangle$. However, there is no easy way to reverse the merge of SCMs back into their constituent components, at least not without allowing one to arbitrarily split an SCM at any endogenous variable. In contrast if we utilise SCMs with socket variables (see Fig. \ref{fig:scm_merge_socket}), we merge SCMs $\langle U^H=\{u,v\}, U^S=\{\}, V=\{w\}, F, P(U), P(U^S) \rangle$ and $\langle U^{H\prime}=\{\}, U^{S\prime}=\{x\}, V^\prime=\{y,z\}, F^\prime, P(U^\prime), P(U^{S\prime}) \rangle$ to form a new SCM $\langle U^{H\prime\prime}=\{u,v\}, U^{S\prime\prime}=\{x\}, V^{\prime\prime}=\{w,y,z\}, F^{\prime\prime}=F\cup F^\prime, P(U^{\prime\prime}), P(U^{S\prime\prime})) \rangle$. Because this formulation specifically identifies the socket variables --- and importantly maintains their probability distributions --- one can split the new SCM back into its constituent parts if desired.

Lastly, the previously described decomposition leaves us with a set of hidden exogenous variables $U^H$, which for all intents and purposes are normal exogenous variables. However, it is assumed that access to modify $P(U^H)$ is determined by system design following principles of encapsulation, hence they are `hidden' behind the SCM wrapper.

Socket variables are not so much used within the SCM components of Fig. \ref{fig:both_architecture_figures} but between them, functioning as an interface. One can quite easily see that it is possible to substitute $w$ and $x$ in Fig. \ref{fig:scm_merge} for the motor torque / steer output of a \emph{Controller FWD Car} component and input of a \emph{FWD Car} component respectively. Allowing their dynamic reconfiguration at runtime in a way that otherwise would not be possible.

\subsection{Retrospective Causal Stationarity with Mutable Input Sets} \label{subsec:parent_set_variables}
In order to capture a system of multiple agents interacting within a shared environment, it is necessary to be able to support the collation of a dynamic number of sources of input data into a set. It is trivial to construct a structural equation which takes input from a fixed number of agents provided the inputs are known. Thus the challenge is extending SCMs to allow the dynamic introduction and removal of SCM modules corresponding to agents entering and leaving the perceived environment.


The issue with this is that altering the structural equations of SCMs amounts to a breach of causal stationarity \citep{runge2018causal}, which poses an issue for window graph roll-outs, our own PTS variable temporal representation, as well as many inference techniques applied to SCMs. However, we can demonstrate that online mutation of an input set to a structural equation can be supported while maintaining a form of causal stationary:

\begin{definition}[Retrospective Causal Stationarity]
Let $T$ be defined in such a way that $\min T = 0$. A fixed SCM $M_t$ that accurately models causal relations for time $t \in T$ has retrospective causal stationarity (RCS) if it accurately models all previous time steps $T_{<t} = \{\,t^\prime\,|\,t^\prime < t,\,t^\prime \in T\,\}$. This effectively means that if $t$ is the most recently observed time step, $M_t$ can be utilised for inference without concern for a lack of causal stationarity.
\end{definition}

Assuming the most up to date SCM in use has RCS we can safely utilise the SCM across all past time steps while remaining faithful to the underlying data. Hypothetically it is also possible to use such an SCM for future time step prediction, but such inferences suffer the same caveats that most extrapolations do. In any case, App. \ref{app:mutable_input_set_proof} provides proof that it is possible to support the dynamic introduction and removal of set elements as input for structural equations. This provides a well defined process by which one can model a varying number of agent interactions via SCMs, a valuable trait for AVs operating in a shared environment. This also importantly demonstrates the utility of socket variables and time conditional variables in enabling this process to occur. This is primarily used within the components of Fig. \ref{fig:both_architecture_figures} to model the varying number of \emph{Link} interactions between \emph{FWD Car} components, each representing an agent on the road.

\section{Results \& Discussion} \label{sec:discussion}

While this paper offers first and foremost theoretical contributions we nonetheless wish to show the practical outcomes such work offers. Thus in addition to offering our code\,\footnote{\href{https://github.com/cognitive-robots/av-extended-scms-paper-resources}{\url{https://github.com/cognitive-robots/av-extended-scms-paper-resources}}} we also revisit the scenario depicted in Fig. \ref{fig:scenes} and Fig. \ref{fig:architecture_scenario}.

In order to do this, we consider tracks from the vehicular \emph{highD} dataset \citep{krajewski2018highd} and intervene upon the actions of one agent such that it ends up colliding with another agent. Through this we are able to produce a crash scene, which we can then post-hoc analyse. We do this by modulating each potential cause action and examine how this influences the presence of a crash. For further details on exactly how we approach the task of identifying causal links between agent actions, we refer the reader to the following literature \citep{howard2025generating} as this methodology is not the focus of this particular work. From here we discuss the qualitative results of applying this methodology to a highD scene and where our extensions have been utilised.

\begin{figure}[t]
    \centering

    \subfigure[Observed scene with collision]{
    \centering
    \includegraphics[width=0.235\textwidth]{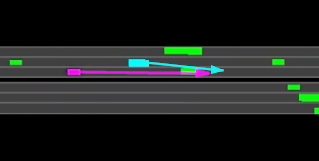}
    \includegraphics[width=0.235\textwidth]{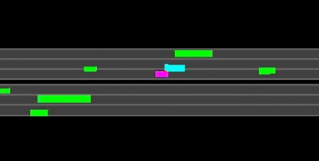}
    } \label{fig:scene_both_true}
    \subfigure[Counterfactual, no magenta brake]{
    \centering
    \includegraphics[width=0.235\textwidth]{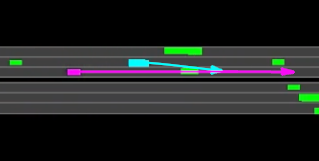}
    \includegraphics[width=0.235\textwidth]{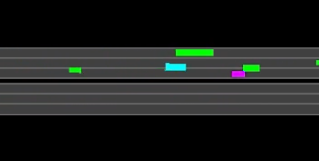}
    } \label{fig:scene_magenta_true}
    
    \subfigure[Counterfactual, no cyan merge]{
    \centering
    \includegraphics[width=0.235\textwidth]{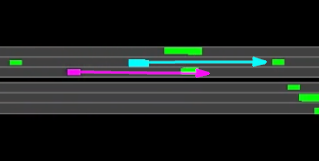}
    \includegraphics[width=0.235\textwidth]{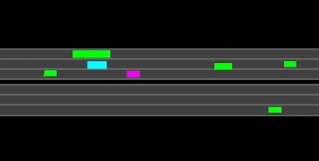}
    } \label{fig:scene_cyan_true}
    \subfigure[Counterfactual, neither action]{
    \centering
    \includegraphics[width=0.235\textwidth]{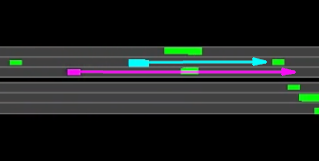}
    \includegraphics[width=0.235\textwidth]{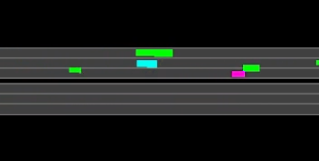}
    } \label{fig:scene_neither_true}
\vspace{-5mm}
    \caption{Visualisation of crash scene within SCM architecture. Subfigures show the initial scene with plan on the left and the final outcome on the right.}
    \label{fig:scenes_true}
\end{figure}

In Fig. \ref{fig:scenes_true} one can see a scenario analogous to the one shown in Fig. \ref{fig:scenes}, only this one is generated from data drawn from the \emph{highD} dataset and analysed counterfactually through the presented SCM architecture. In this scenario the cyan vehicle attempts to merge into a lane after magenta has begun braking, resulting in a collision. Note that this scenario differs from the one described in Sec. \ref{sec:example_scenario}. This was a deliberate decision in order to demonstrate the flexibility of the architecture in not being constrained to one particular scenario, as has been the case in some previous works \cite{maier2024causal}. 

By modulating the actions of each agent via counterfactual interventions as described above we can attempt to identify the culpable agent. Conveniently, because the agent actions each relate to specific points in time we can rely upon the properties of temporal precedence and the existence of reaction times to avoid the formation of any cycles in the resulting causal graph. This does of course depend upon the length of time steps between data measurements being less than what one might expect in terms of a reaction time. For \emph{highD} this is $40\,ms$, far below the average reaction time of a human --- i.e. $\sim220\,ms$ \citep{laming1968information} --- avoiding any concern.

Through these simulations we can infer that not only would the crash not have occurred had cyan not attempted to merge, but if magenta had not braked then it would have collided with another vehicle instead. Thus at least from a decision-making standpoint, cyan appears to carry greater culpability than magenta within the scenario.

Regarding how this example utilises the paper's contributions, consider the following:
\begin{itemize}
    \item Use of PTS variable temporal representation combined with sparse data storage with buffer variables keeps memory storage for a large number of agents tractable.
    \item In a similar manner, when processing a large number of counterfactual queries potentially using multiple different controllers and planners, socket variables make the reuse of SCM modules easier, and time conditional variables enable the reuse of data and management of multiple timelines.
    \item The retrospective causal stationarity proof for mutable input sets, combined with socket variables and time conditional variables, allow for a dynamically changing number of elements to be represented across a timeline --- in this case AV agents.
\end{itemize}
While each one of these contributions might offer a smaller contribution, their gestalt offers many new tools that can be used within the AV space and beyond. We argue that in addition to future work making use of the causal formalisms and code presented here, it would be beneficial to explore integrating causality into other autonomous embodied systems.

\subsection{Conclusion}
To summarise, the presented extensions represent a step forward for both the expressive capabilities of SCMs within causality research as well as work integrating AVs with causality. In building such systems we continue the push towards autonomous systems that are transparent, accountable, and understandable. Through these traits we can best ensure that the technologies of tomorrow are developed responsibly.

\acks{This work was supported by the EPSRC project RAILS (grant reference: EP/W011344/1) and the Oxford Robotics Institute research project RobotCycle.}

\bibliography{references}

\appendix

\section{Glossary of Technical Terms}
Given the cross-disciplinary nature of this work, there may be terms present which are not familiar for researchers primarily focused upon causality. Thus for the benefit of the reader we provide the following short summaries of technical terms used within the paper.

\subsection{General Computing Theory}
\begin{itemize}
    \item Space Complexity: Describes how the memory or storage requirements of a function varies with its input parameters.
    \item Tractable: Describes a function can feasibly be executed for given set of inputs based upon the associated time and data memory / storage requirements.
    \item Memoryless: In the context of discrete time-indexed states, a state can be described deterministically or stochastically purely in terms of the previous state.
    \item Recursion: Describes when a function calls itself as part of its operation.
    \item Meta-Variable: A variable that exists outside of the models / systems where it is used to define or implement basic functionality. Typically utilised in an auxiliary role, e.g. the number of threads a computation uses will vary how it is run, but ultimately the results of the computation will be the same.
\end{itemize}

\subsection{Robotics / Autonomous Vehicles}
\begin{itemize}
    \item Agent: An independent decision maker, e.g. a human, a robot.
    \item Embodied Agent: An agent that has a physical presence with which it can interact in and with the world. For example, AI agents that purely analyse existing data --- such as those used in the medical sector to make treatment decisions --- are not embodied.
    \item Grid-World Problems: A type of task or problem-solving scenario in which the world is simplified into a grid. Within such a world states and actions are also defined in terms of the grid. This is often done to simplify a task, as transforms the state and action set into discrete sets which may be easier to work with.
    \item Advanced Driver Assistance System: Commonly known by its acronym `ADAS' and often associated with automation levels 1 and 2 as defined by the \cite{sae2021taxonomy}, an ADAS provides some level of automation while still requiring the driver to be aware and ready to assume full control. Most frequently this involves maintaining a desired speed while avoiding collisions and remaining within speed limits, although newer systems also manage steering and lane control.
    \item Front-Wheel Drive: Describes a vehicle in which the motor is connected to and delivers torque through the front wheel axel. Note that this acronym is not defined in the main text due to all instances 
    \item Kalman Filtering: A technique for estimating the values of noisy variables across time \citep{kalman1960new}. It does this by incorporating models that describe how the state changes across time, how values the system itself measures or senses values based upon the true values, and the noise associated with both of these processes. Optionally it can also consider how actuation values affected by the system affect the true values.
    \item Goal Value: Within the context of this paper, goal values comprise actions and correspond to sensory and / or internal variables, for which they provide a target. These are typically associated with more high-level variables such as speed or lane in the context of AVs.
    \item Actuation Value: These are values through which an agent interacts with its embodied form.  These are typically associated with more low-level variables. In the context of AVs these might be something like motor torque or steering.
    \item Controller: Takes one or more high-level goal values and varies actuation values in response in an attempt to bring sensory and / or internal variables closer to their corresponding goal values.
    \item Planner: Considers the state of the embodied agent and the world and decides upon some actions, typically based upon some measure of reward --- i.e. how beneficial an action or state is from the perspective of a given agent --- or overall objective.
\end{itemize}

\subsection{Software Engineering}
\begin{itemize}
    \item Monolithic: Describes software architectures that compile all functionality into a single comprehensive unit.
    \item Modularisation: A software design technique that advocates splitting code into several modules that each handle a particular aspect of the overall system functionality.
    \item Encapsulation: A software design technique that advocates storing data along with the functions that act upon said data within the same structures. This approach also typically suggests that access to said data should only be possible via the same aforementioned functions.
    \item Wrapper: When speaking of encapsulation, a wrapper describes the structure that encapsulates a specific set of data elements. This is particularly the case when discussing access control for said data, which is typically managed through some form of interface the structure has. Note that this has nothing to do with wrapping in the context of monads.
    \item Self-Extension: This describes when a system is able to extend its own functionality or capabilities based upon data acquired over the course of its operation. Note that this is different to, but not exclusive from machine learning. For example, machine learning could be applied to help predict disease, and while it may achieve this for the target disease, it will not attempt to start learning to predict other diseases. By contrast, an AV might be programmed in such a way that it learns to operate within environments other than those for which it was engineered, thus demonstrating self-extension.
\end{itemize}

\subsection{Functional Programming Monads}
\begin{itemize}
    \item Type Constructor: Takes a selection of input types and creates a new type as output. For example, a list type constructor could take the integer type as input and produce the integer list type as output.
    \item Side-Effect: Computations that occur within a function outside of purely deriving the output given the input values. These can correspond to file input / output, random number generation, and more.
    \item Monad: Within functional programming monads are able to store additional information besides just the values passed to and from functions. Hence this additional information is `wrapped' around values as they are passed between functions. While they have a variety of uses, one use is to have them represent side-effects in order to facilitate such features within an ecosystem otherwise comprised of deterministic functions. In the context of this paper, these deterministic functions are the endogenous variables of SCMs, which can utilise side-effects in order to provide additional functionality to the overall system. There are several more formalised aspects of monads, however these are detailed in the main text.
    \item Monadic Value: This describes the value returned after wrapping an original value in a monad. This may  consist purely of additional information, or it could have additional computation associated with it too.
\end{itemize}

\section{PTS Variable Recursion Proof} \label{app:pts_variable_recursion_proof}

Although the concept of PTS variables was introduced in the main text, we aim to illustrate here why the use of PTS variables, and the form of recursive structure they rely upon, does not pose an obstacle to utilising existing techniques from literature. Consider Fig. \ref{fig:scm_structures_rollout}, in which the two SCMs from Fig. \ref{fig:scm_structures} are rolled out for time steps 0--2 inclusive. Given that PTS variables are only responsible for managing $C_T$ and buffer variables just cache values indexed to $C_T$ --- in other words fulfilling utility roles --- the structure of the roll-outs are otherwise identical. Thus any series of operations one could perform on Fig. \ref{fig:time_series_scm_rollout} could be mapped directly to \ref{fig:system_scm_rollout} without issue. Ultimately the key difference between these roll-outs is that in Fig. \ref{fig:time_series_scm_rollout} each variable is replicated across time steps based upon the window graph SCM, whereas in Fig. \ref{fig:system_scm_rollout} the variables are re-used. In essence, the extended SCM structure utilising PTS variables virtually emulates the replication of variables by maintaining the current time step as a meta-variable $C_T$, which can be utilised by buffer variables to time-index data. The key benefit this offers is not only in avoiding having to actually replicate the original SCM variables, but being able to maintain the same structure --- as indicated by the double-ended arrow --- which may be beneficial for certain system designs.

\begin{figure}[t]
    \centering

    \subfigure[Window Graph SCM]{
    \includegraphics[width=0.8\linewidth]{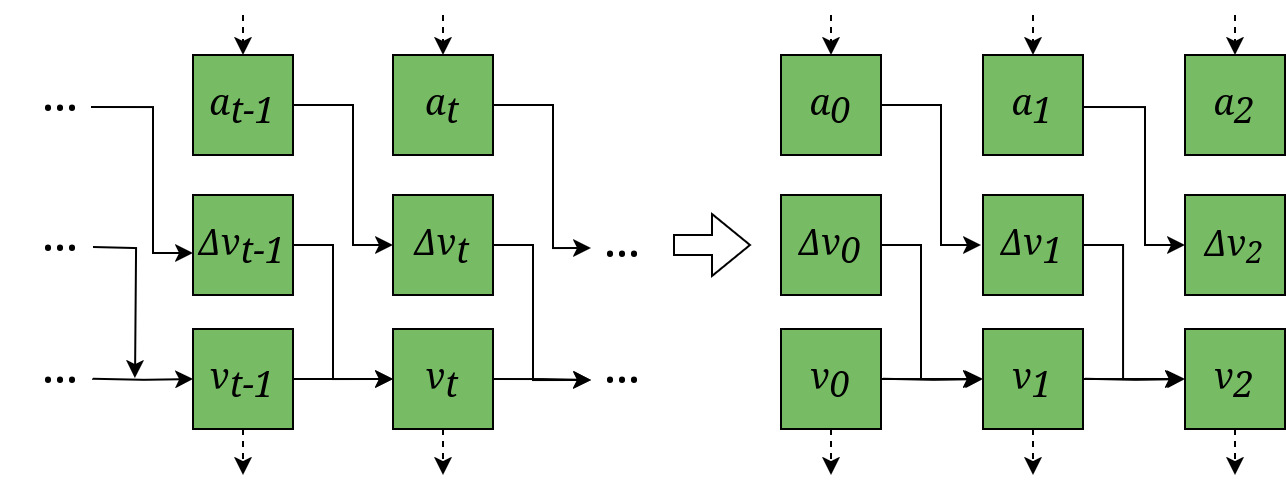}
    } \label{fig:time_series_scm_rollout}
    \\
    \subfigure[Extended SCM]{
    \includegraphics[width=\linewidth]{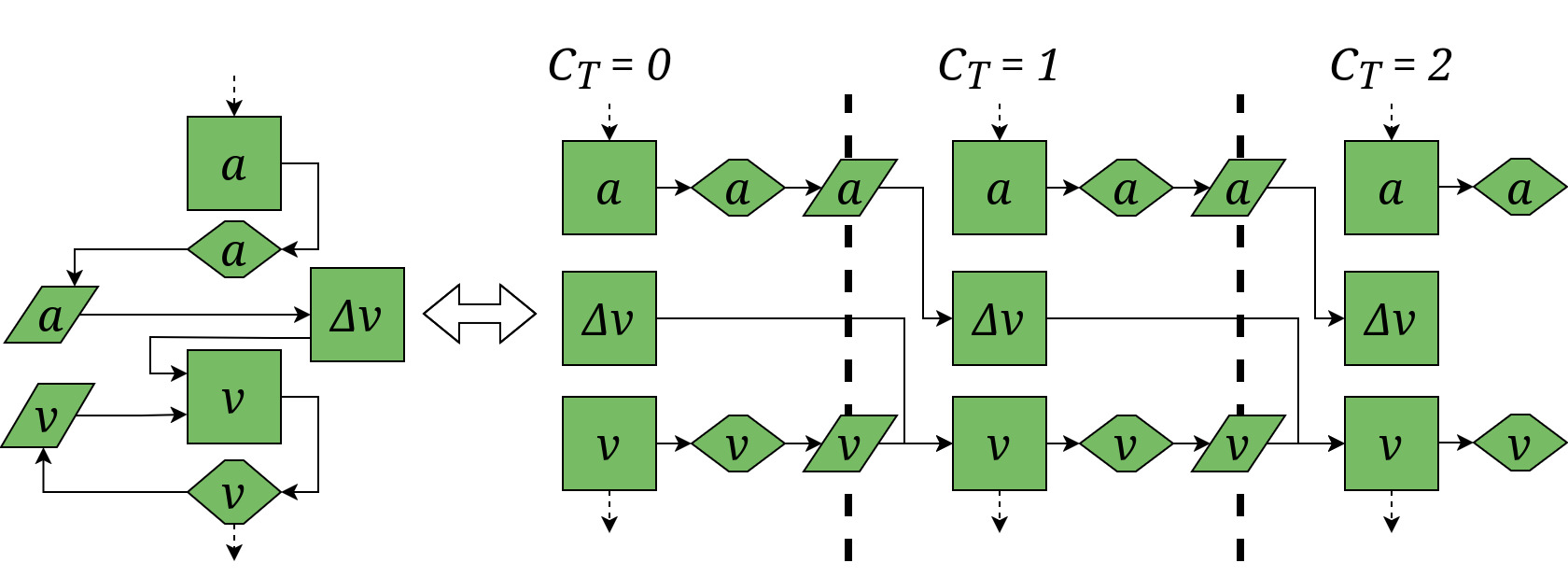}
    } \label{fig:system_scm_rollout}

    \caption{Depicts the rollout of the SCMs in Fig \ref{fig:scm_structures} for time steps 0--2 inclusive.}
    \label{fig:scm_structures_rollout}
\end{figure}

\section{Mutable Input Set RCS Proof} \label{app:mutable_input_set_proof}
Here we define a theorem that supports the use of a dynamic set of inputs for a structural equation:

\begin{theorem}[Mutable Input Set RCS] \label{theorem:parent_set_variables}
One can capture the dynamic introduction and removal of sources of type $\mathbb{N} \times T$ used together as an input set for a variable $V_i$ with structural equation $f_{V_i}(\cdot) : 2^{\mathbb{N} \times T} \rightarrow \mathcal{T}_{V_i}$ via series of SCMs $\{ M_0, M_{\delta}, ..., M_t \}$ that each provide retrospective temporal stationary. 
\end{theorem}

Here the set of natural numbers $\mathbb{N}$ is used to uniquely identify contributions from different sources (e.g. agents) and importantly allows the same value of type $T$ to be captured several times within the input set.

We first describe the structure utilised within the SCM in order to construct sets from variable outputs. The input to $V_i$ is incrementally built up from the combination of: 
\begin{itemize}
    \item A socket variable $U^S_\emptyset$ where the probability distribution $P(U^S_\emptyset)$ is a degenerate distribution that always returns the empty set $\emptyset$.
    \item A union variable $V_\cup$ that has two parents providing input sets of type $2^{\mathbb{N} \times T}$ and outputs the union of said sets. The first of these parents provides a singleton set from a given source (e.g. an agent), while the second parent is a socket variable as described above.
\end{itemize}
This effectively allows the system to iteratively construct a set beginning with $U^S_\emptyset$ providing an empty set. From here one can connect a union variable $V^\prime_\cup$ in order to incorporate a new input into the set, while providing its own socket variable $U^{S\prime}_\emptyset$ from which additional another union variable could be connected in future. This structure allows for the dynamic introduction and removal of input sources, as it effectively emulates a linked list.

We can now consider the following basic lemma:

\begin{lemma}[Base Case] \label{lemma:first_model}
Provided the SCM $M_0$ at the first time step is accurate for said time step, it has retrospectively causal stationarity.
\end{lemma}
\begin{proof}[Proof of Lemma \ref{lemma:first_model}]
Given that $M_0$ accurately models causal relations for $t = 0$, and $T_{<0} = \emptyset$, it automatically follows that $M_0$ accurately models time steps in $T_{<0}$.
\end{proof}

So that we may capture the dynamic nature of the input set sources, rather than having an input source feed directly into a union variable $V_\cup$, it should instead feed through a pair of time conditional variables $V^\alpha_{T?}$ and $V^\omega_{T?}$. These variables should only return the input set source for $\theta^\alpha_T \leq T \leq \theta^\omega_T$ where $\theta^\alpha_T$ and $\theta^\omega_T$ are configurable time parameters for each of the aforementioned time conditional variables respectively. Should $t$ fall outside of these bounds, the time conditional variables should return the output of a fixed exogenous variable $U_\emptyset$ where $P(U_\emptyset)$ is a degenerate distribution that always returns the empty set $\emptyset$.


\begin{lemma}[Introduction Induction Step] \label{lemma:introduction}
If $M_{t-\delta}$ is an SCM with RCS and at least one new input set source is introduced at the next time step $t$ then one can construct a new SCM $M_t$ that accurately captures causal relationships at time step $t$ and has RCS.
\end{lemma}
\begin{proof}[Proof of Lemma \ref{lemma:introduction}]
Construct $M_t$ by extending the SCM to incorporate the new input set sources as described above, assign $\theta^\alpha_T = t$ and $\theta^\omega_T = t$ for their time conditional variable parameters. Based upon the configuration described above, $\theta^\alpha_T \leq t \leq \theta^\omega_T$ and thus $M_t$ will include the new input set source, accurately capturing the causal links at time $t$. 

Meanwhile, given that $T_{<t}$ is defined as all time steps less than $t$, and now $\theta^\alpha_T = t$, we can conclude that for all previous time steps captured by $T_{<t}$, the new SCM extension will only contribute $\emptyset$ to the input set. This emulates the input set source not being present in the past, accurately reflecting the causal links present during those time steps. Thus, we can conclude that $M_t$ also has RCS after the new input set source introduction.
\end{proof}

\begin{lemma}[Status Quo / Removal Induction Step] \label{lemma:removal}
If $M_{t-\delta}$ is an SCM with RCS and zero or more input set sources are removed at the next time step $t$ then one can construct a new SCM $M_t$ that accurately captures causal relationships at time step $t$ and has RCS.
\end{lemma}
\begin{proof}[Proof of Lemma \ref{lemma:removal}]
At each time step $t$, for each input set source that is present in the modelled system, assign $\theta^\omega_T = t$ for the configurable parameter of the associated $V^\omega_{T?}$ in $M_t$ respectively. SCM input set sources for which $\theta^\alpha_T \leq t \leq \theta^\omega_T$ have their contributions included in the input set. Thus, this update combined with Lemma \ref{lemma:introduction} ensures that all input set sources that are present in the modelled system will have their contributions captured within the input set, as is accurate for time step $t$.

All the while, $\theta^\omega_T$ is always first initialised to the time step in which the associated input set source was introduced --- as per Proof of Lemma \ref{lemma:introduction} --- and is only updated as described above. As such, if an update is missed for a given input set source --- i.e. due to it no longer being present in the modelled system --- then it follows that $t > \theta^\omega_T$, and thus the SCM input set source will return $\emptyset$, no longer contributing to the input set. In doing so the SCM remains accurate for time step $t$, and additionally retains the same behaviour it possessed for previous time steps as $\theta^\alpha_T$ and $\theta^\omega_T$ remain the same as in $M_{t - \delta}$. Thus, we can conclude that $M_t$ also has RCS for instances where the status quo is maintained, and for instances where one or more input set sources are removed.
\end{proof}


\begin{proof}[Proof of Theorem \ref{theorem:parent_set_variables}]
We have demonstrated through Lemma \ref{lemma:first_model} that the accurate SCM associated with the first time step $M_t$ automatically has retrospective temporal stationarity, effectively establishing a base case. Additionally we have shown that input set source introductions (Lemma \ref{lemma:introduction}), removals (Lemma \ref{lemma:removal}), or maintaining the status quo of input set sources between time steps (Lemma \ref{lemma:removal}) can be done while maintaining RCS for $M_t$ given that $M_{t - \delta}$ had it. Thus via induction we can infer that the sequence of SCMs $\{ M_0, M_{\delta}, ..., M_t \}$ accurately captures these dynamic mutations and all possess RCS.
\end{proof}

\section{SCM Architecture Details}
Here the details of the proposed SCM architecture are laid out in addition to detailing the links between the architecture and the formalisms introduced within this work. A legend describing the elements comprising the following SCM diagrams is presented in Fig. \ref{fig:legend}. For the sake of brevity in an already lengthy appendix, we refer the reader to the Git repository mentioned in Sec. \ref{sec:discussion} for details of the exact structural equations utilised by endogenous variables.

For the AV domain considered by this work goals have been formulated as a combination of a target lane and a target speed with corresponding times by which to achieve these goals. These goals are converted into motor torque and steering actuation values via a proportional feedback controller. The motor torque and steering values are then fed into a dynamic bicycle model \citep{guiggiani2018science}. Meanwhile the entity objects and their associated link objects are used to provide agent collision and drag forces / torques, treating vehicles as rectangular rigid bodies. The mechanics of the vehicles also assume they use front-wheel drive.

    \begin{figure}[t]
        \centering
        \includegraphics[width=0.3\linewidth]{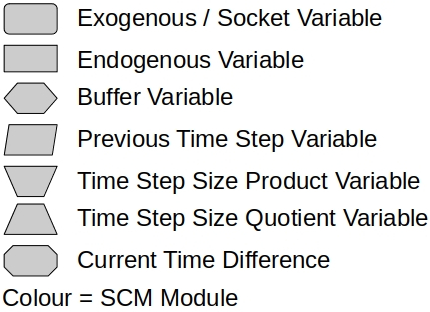}
        \caption{SCM Legend}
        \label{fig:legend}
    \end{figure}

    \paragraph{Point Mass}
    \begin{figure}[t]
        \centering
        \includegraphics[width=0.6\linewidth]{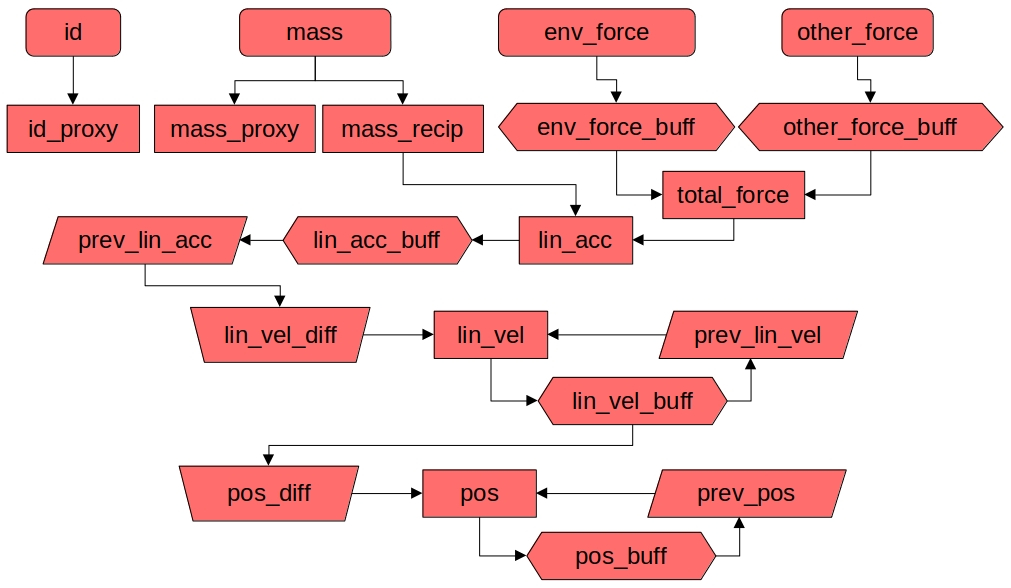}
        \caption{\emph{Point Mass} SCM}
        \label{fig:point_mass_scm}
    \end{figure}

    The SCM depicted in Fig. \ref{fig:point_mass_scm} represents a 2D dynamic point mass \citep{blum2006mathematics} that does not explicitly occupy any portion of space yet provides the represented object with a mass, position, linear velocity, linear acceleration, and input forces which influence these. 

    Within the context of the example scenario, this is utilised to model the basic dynamics of the vehicles. While this is certainly not something that is only possible via the extensions provided here, the extensions do provide the following benefits:
    \begin{itemize}
        \item Succinct temporal representation through the \emph{prev\_pos}, \emph{prev\_lin\_vel}, and \emph{prev\_lin\_acc} PTS variables, preventing the need to duplicate variables across time steps.
        \item Encapsulation of SCM data within \emph{pos\_buff}, \emph{lin\_vel\_buff}, \emph{lin\_acc\_buff}, \emph{env\_force\_buff}, and \emph{other\_force\_buff} buffer variables.
        \item Utilisation of \emph{env\_force} and \emph{other\_force} socket variables, allowing the dynamic reconfiguration of force inputs from the environment, and other forces (e.g. self propelled motion).
    \end{itemize}

    \paragraph{Rectangular Rigid Body}

    \begin{figure}[t]
    \centering
    \includegraphics[width=0.8\linewidth]{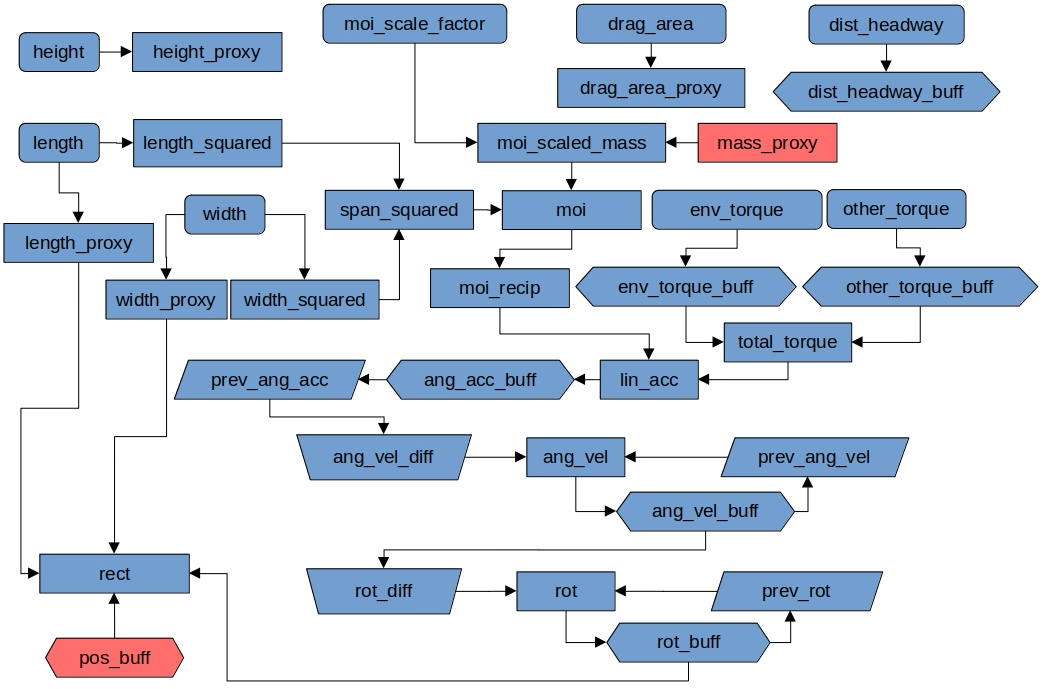}
    \caption{\emph{Rectangular Rigid Body} SCM}
    \label{fig:rectangular_rigid_body_scm}
    \end{figure}

    The SCM depicted in Fig. \ref{fig:rectangular_rigid_body_scm} inherits from the \emph{Point Mass} SCM, extending it to represent a 2D rectangular rigid body \citep{blum2006mathematics}. In doing so it allows the object to occupy a rectangular portion of space and provides it with a moment of inertia, rotation, angular velocity, angular acceleration, and input torques which influence these. 
    This SCM additionally tracks the open space in-front of the rectangular rigid body, as this information is calculated via the \emph{Rectangular Rigid Body Entity} SCM via which a rectangular rigid body is represented in a shared environment.

    Similar to the \emph{Point Mass} SCM, this is just used to provide further dynamics to the vehicles within the example scenario. Again, utilisation of the extensions is not strictly necessary here, but provides the following benefits:
    \begin{itemize}
        \item Succinct temporal representation through the \emph{prev\_rot}, \emph{prev\_ang\_vel}, and \emph{prev\_ang\_acc} PTS variables, preventing the need to duplicate variables across time steps.
        \item Encapsulation of SCM data within \emph{rot\_buff}, \emph{ang\_vel\_buff}, \emph{ang\_acc\_buff}, \emph{dist\_headway\_buff}, \emph{env\_torque\_buff}, and \emph{other\_torque\_buff} buffer variables.
        \item Utilisation of \emph{dist\_headway}, \emph{env\_force} and \emph{other\_force} socket variables, allowing the dynamic reconfiguration of force / distance headway inputs from the environment, and other forces.
    \end{itemize}

    \paragraph{Rectangular Rigid Body Entity}
    
    \begin{figure}[t]
        \centering
        \includegraphics[width=0.6\linewidth]{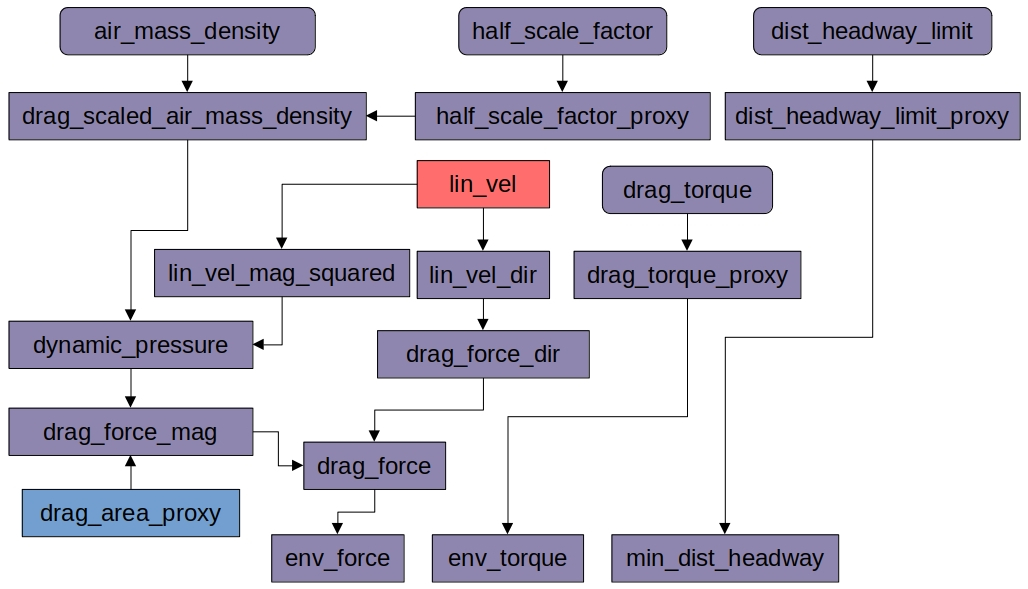}
        \caption{\emph{Rectangular Rigid Body Entity} SCM}
        \label{fig:rectangular_rigid_body_entity_scm}
    \end{figure}
    
    The SCM depicted in Fig. \ref{fig:rectangular_rigid_body_entity_scm} represents rectangular rigid body objects within a shared environment. Performs two functions. The first is to calculate environmental forces, combining drag force and the collision forces calculated by the \emph{Rectangular Rigid Body Link} SCMs associated with a given \emph{Rectangular Rigid Body Entity} SCM. The second is to calculate the minimum distance headway given by the aforementioned \emph{Rectangular Rigid Body Link} SCMs in order to determine the overall distance headway.

    This SCM facilitates all interactions between a given vehicle and the shared environment within the example scenario. Unlike the previous two SCMs, this SCM inherently relies upon the formalisms introduced here:
    \begin{itemize}
        \item The structure used to demonstrate that mutable input sets can be facilitated while maintaining RCS is used here in order to sum up the various environmental forces / torques that affect a given vehicle, as well as calculate the minimum distance headway. In particular this allows vehicles to come and go throughout the lifetime of a scenario such as the one described above. The implementation of this structure in turn requires the use of time conditional and socket variables.
        \item Due to the SCM not utilising any buffer variables, this SCM does not store any data associated with its variables, helping to optimise memory usage.
    \end{itemize}
    
    \paragraph{Rectangular Rigid Body Link}

    \begin{figure}[t]
        \centering
        \includegraphics[width=0.7\linewidth]{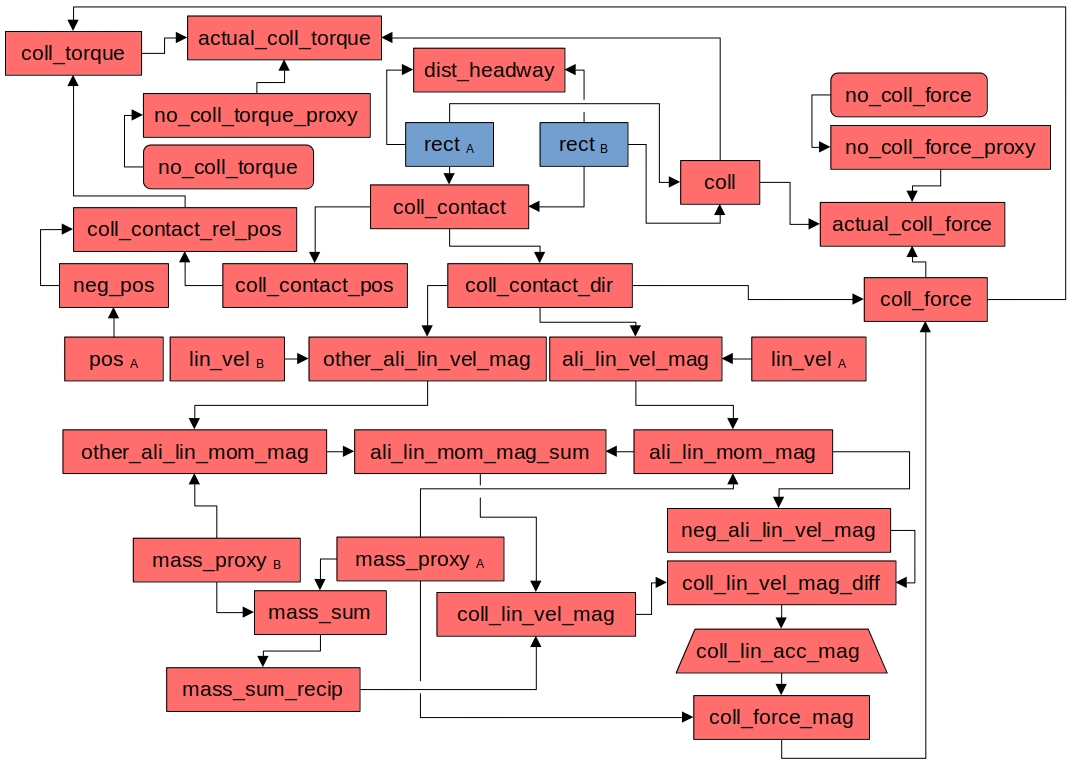}
        \caption{\emph{Rectangular Rigid Body Link} SCM}
        \label{fig:rectangular_rigid_body_link_scm}
    \end{figure}

    The SCM depicted in Fig. \ref{fig:rectangular_rigid_body_link_scm} captures interactions between two rectangular rigid bodies in a shared environment. Most critically performs the collision computations for the two rigid bodies in question. Additionally computes distance headway between the two rigid bodies as this is frequently used within the driving domain as a safety metric \citep{nahata2021assessing}. Variables belonging to the primary and secondary \emph{Rectangular Rigid Body} SCMs associated with the link will be indicated via $A$ and $B$ subscripts respectively.

    While the \emph{Rectangular Rigid Body Entity} SCM represented the sum total of factors affecting a given vehicle, this SCM captures the interactions between two vehicles in particular, such that SCMs of this type can then have their values collated by a \emph{Rectangular Rigid Body Entity} SCM. This SCM makes a small amount of usage of the extensions:
    \begin{itemize}
        \item The \emph{coll\_lin\_acc\_mag} TSSQ variable allows the calculation of the required acceleration resulting from a collision given a required velocity change and the size of a time step.
        \item Due to the SCM not utilising any buffer variables, this SCM does not store any data associated with its variables, helping to optimise memory usage.
    \end{itemize}
    
    \paragraph{FWD Car}

    \begin{figure}[t]
        \centering
        \includegraphics[width=\linewidth]{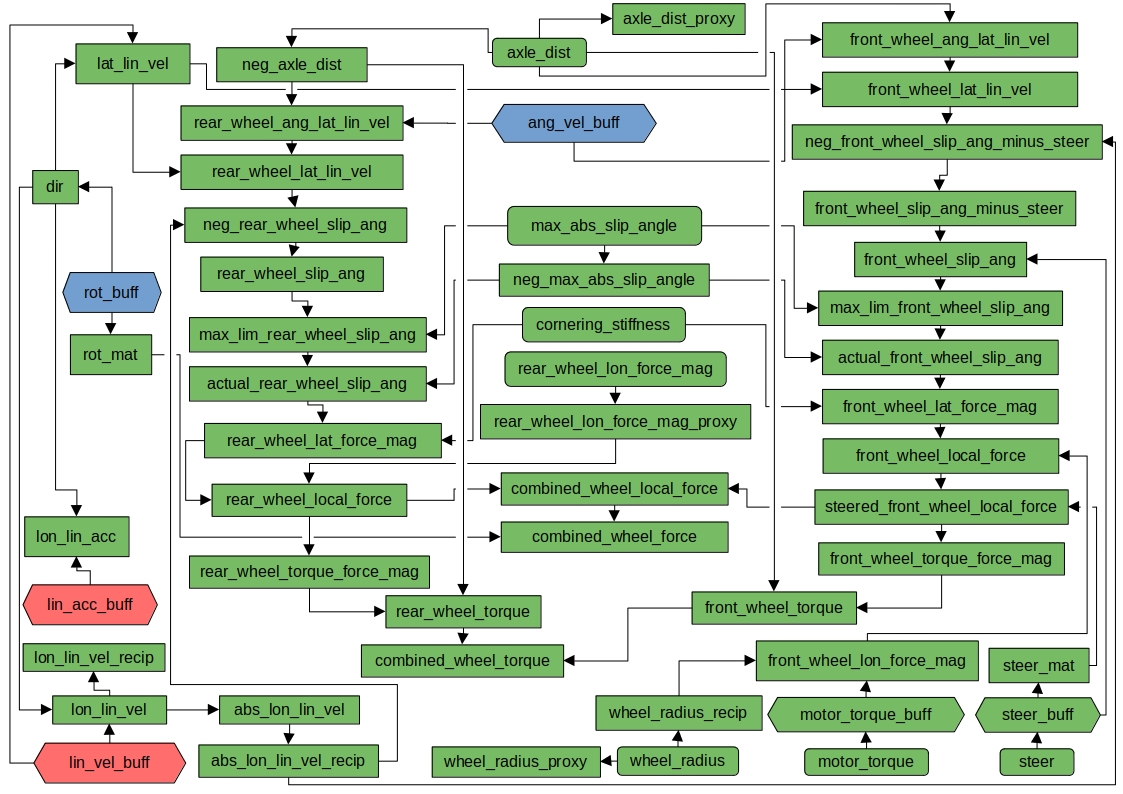}
        \caption{\emph{FWD Car} SCM}
        \label{fig:fwd_car_scm}
    \end{figure}
    
    The SCM depicted in Fig. \ref{fig:fwd_car_scm} inherits from the \emph{Rectantular Rigid Body} SCM to represent a road vehicle, in particular one that has FWD. This SCM provides forces and torques as output to the underlying \emph{Rectangular Rigid Body} SCM, instead providing variables for motor torque and steering as input. In order to calculate the forces and torques required, the SCM utilises these two inputs and treats the represented object as a dynamic bicycle \citep{guiggiani2018science}, a simple means by which one can model a variety of vehicles.
    
    In contrast to the \emph{Point Mass} and \emph{Rectangular Rigid Body} SCMs, this SCM provides the highest level of the physical representation of vehicles within the AV domain and is comprised of vehicle-specific variables. Once again, utilisation of the extensions is not completely necessary here, but is desirable:
    \begin{itemize}
        \item Encapsulation of SCM data within \emph{motor\_torque\_buff} and \emph{steer\_buff} buffer variables. This also saves on the memory consumption of the SCM given the large number of variables.
        \item Utilisation of \emph{motor\_torque} and \emph{steer} socket variables, allowing the dynamic reconfiguration of force / distance headway inputs from the environment, and other forces. This is arguably of greater use here than in the \emph{Point Mass} and \emph{Rectangular Rigid Body} SCMs as this SCM interfaces directly with controller components within the architecture.
    \end{itemize}

    \paragraph{Motor Torque Control FWD Car}

    \begin{figure}[t]
        \centering
        \includegraphics[width=0.65\linewidth]{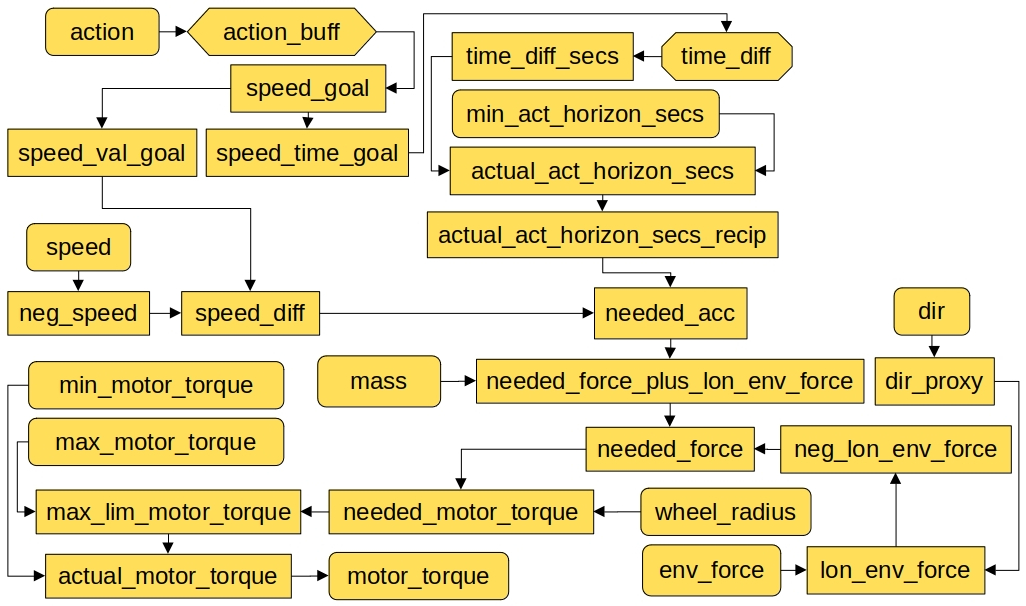}
        \caption{\emph{Motor Torque Control FWD Car} SCM}
        \label{fig:motor_torque_control_fwd_car_scm}
    \end{figure}
    
    The SCM depicted in Fig. \ref{fig:motor_torque_control_fwd_car_scm} provides a control mechanism which outputs motor torque in return for giving a goal speed and a goal time as input. From the difference between the longitudinal linear velocity given by the \emph{FWD Car} SCM and the goal velocity, and the difference between the current time and the goal time it is possible to calculate the necessary acceleration. Using the mass of the \emph{FWD Car} SCM one can calculate the required force, before making adjustments to the force to account for environment forces also acting upon the \emph{FWD Car} SCM. From the calculated force and wheel radius of the vehicle one can finally approximate the motor torque that can then be fed into the \emph{FWD Car} SCM.

    Within the context of the example scenario this SCM, along with the \emph{Steer Control FWD Car} SCM, is responsible for bridging the gap between physical representation and action planner for the vehicles in question. Here, there is utilisation of several SCM formalism extensions:
    \begin{itemize}
        \item The \emph{time\_diff} TCTD variable allows the calculation of the time difference between the action goal time and the current time.
        \item Encapsulation of SCM data within \emph{action\_buff} buffer variable.
        \item Utilisation of \emph{action} socket variable, allowing the dynamic reconfiguration of which planner ought to be used.
        \item Utilisation of \emph{speed}, \emph{dir}, \emph{mass}, \emph{wheel\_radius}, and \emph{env\_force} socket variables. Unlike the \emph{action} socket variable, these are used as inputs from the physical representation, which are utilised to determine how the driving actuation variables of the physical representation (i.e. the \emph{motor\_torque} variable) ought to be modulated.
    \end{itemize}
    
    \paragraph{Steer Control FWD Car}

    \begin{figure}[t]
        \centering
        \includegraphics[width=0.65\linewidth]{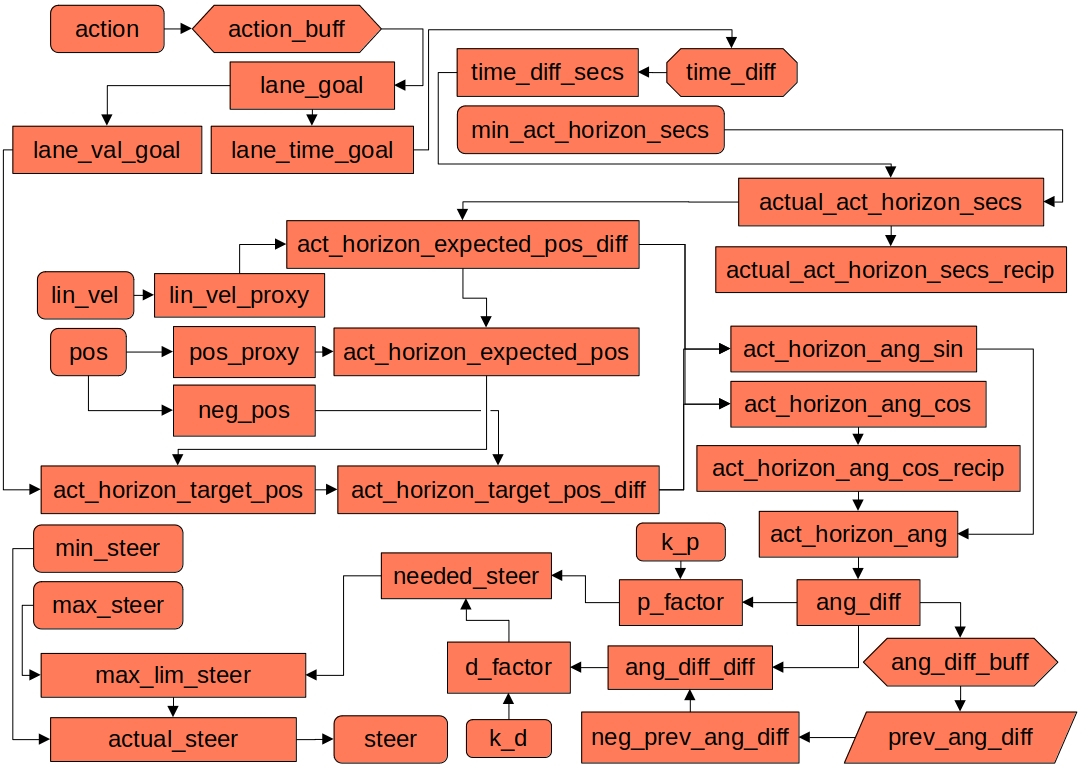}
        \caption{\emph{Steer Control FWD Car} SCM}
        \label{fig:steer_control_fwd_car_scm}
    \end{figure}

    The SCM depicted in Fig. \ref{fig:steer_control_fwd_car_scm} provides a control mechanism which outputs steer in return for giving a goal lane id and a goal time as input. In order to calculate the steering first the expected position of the rectangular rigid body is calculated for a predefined future horizon. This is then projected onto the closest point of the central path of the lane associated with the goal lane id. The expected position combined with the projected position and current position given by the \emph{FWD Car} SCM provide an angle error. This angle error, along with its derivative are then used as inputs to a proportional-derivative controller \citep{liptak2013process}, with the steering as the output which is then fed into the \emph{FWD Car} SCM.
    
    Within the context of the example scenario this SCM, along with the \emph{Motor Torque Control FWD Car} SCM, is responsible for bridging the gap between physical representation and action planner for the vehicles in question. Here, there is utilisation of several SCM formalism extensions:
    \begin{itemize}
        \item The \emph{time\_diff} TCTD variable allows the calculation of the time difference between the action goal time and the current time.
        \item Encapsulation of SCM data within \emph{action\_buff} and \emph{ang\_diff\_buff} buffer variables.
        \item Utilisation of \emph{action} socket variable, allowing the dynamic reconfiguration of which planner ought to be used.
        \item Utilisation of \emph{pos} and \emph{lin\_vel} socket variables. Unlike the \emph{action} socket variable, these are used as inputs from the physical representation, which are utilised to determine how the driving actuation variables of the physical representation (i.e. the \emph{steer} variable) ought to be modulated.
    \end{itemize}

    \paragraph{Full Control FWD Car}
    Simply a wrapper SCM that inherits from both \emph{Motor Torque Control FWD Car} and \emph{Steer Control FWD Car} SCMs. These SCMs combined are able to take a goal speed and goal lane id, along with target times by which to accomplish these goals, and in turn output motor torque and steer values for the associated \emph{FWD Car} SCM.

    \paragraph{Greedy Plan FWD Car}

    \begin{figure}[t]
        \centering
        \includegraphics[width=0.65\linewidth]{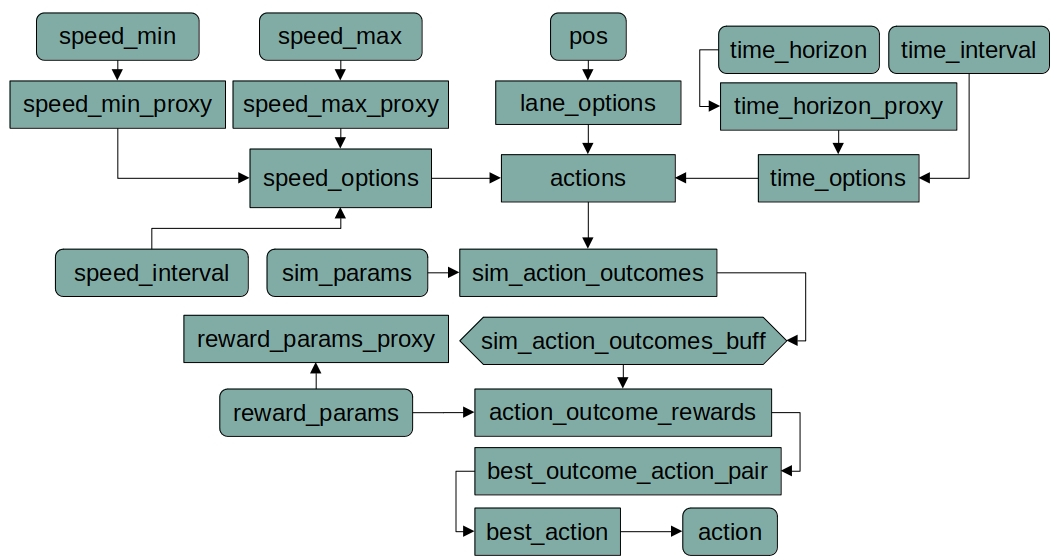}
        \caption{\emph{Greedy Plan FWD Car} SCM}
        \label{fig:greedy_plan_fwd_car_scm}
    \end{figure}
    
    The SCM depicted in Fig. \ref{fig:greedy_plan_fwd_car_scm} plans the next action for a \emph{Full Control FWD Car} SCM specified as goals. In order to do so, a range of candidate speeds, lanes, and target times are calculated and combined to form a variety of possible actions. The outcomes associated with taking these possible actions are then simulated, and the rewards associated with the action-outcome pairs are calculated. This does rely upon predefined a predefined simulator and predefined reward calculator being provided to the \emph{Greedy Plan FWD Car} SCM. In any case, once the rewards have been calculated the action associated with the maximum reward is selected and fed into the corresponding \emph{Full Control FWD Car} SCM.

    Arguably the most important SCM besides the \emph{FWD Car} SCM in the context of the example scenario, this SCM captures the decision making capabilities of vehicular agents within the scene. In other words the decisions that are being identified as part of behavioural interactions are assumed to be made within a roughly comparable form of planner. In order to achieve this the following extensions are utilised:
    \begin{itemize}
        \item Encapsulation of SCM data within \emph{sim\_action\_outcomes\_buff} buffer variable.
        \item Utilisation of \emph{pos} socket variable. This is used as an input from the physical representation, which is utilised to determine how an action ought to be selected. In particular the \emph{pos} variable is used to identify candidate lanes for the vehicular agent.
    \end{itemize}

\end{document}